\documentclass{tlp}
\usepackage{amsmath}
\usepackage[disable]{todonotes}
\usepackage{multicol}
\usepackage{times}
\usepackage{helvet}
\usepackage{courier}
\usepackage{xspace}
\usepackage{latexsym}
\usepackage{enumitem}
\usepackage{verbatim}
\usepackage{amsmath}
\usepackage{amssymb}
\usepackage{color}
\usepackage{tikz}
\usetikzlibrary{shapes,arrows}

\def\citeps#1{(\citeauthor{#1},\citeyear{#1})}

\usepackage{booktabs}
\usepackage{alltt}
\usepackage{caption}
\def\wr{:\sim}

\DeclareMathOperator*{\argmax}{arg\,max}
\DeclareMathOperator*{\argmin}{arg\,min}

\def \prev#1{{#1}^\uparrow}

\def \br#1{[{#1}]}
\def \signp#1{{#1}^+}
\def \signm#1{{#1}^-}

\def \signo#1{{#1}^{-1\cdot}}

\def\level#1{{\lambda({#1})}}
\def\weight#1{{#1}_w}

 \newtheorem{Property}{Property} 
\newtheorem{Proposition}{Proposition}
\newtheorem{Lemma}{Lemma}
\newtheorem{Definition}{Definition}

\def\<{\langle}
\def\>{\rangle}

\def\Int{\ensuremath{\mathit{Int}}}

 \def\cS{\ensuremath{\mathcal{Z}}}

 \def\cP{\ensuremath{\mathcal{P}}}
 
 \def\cW{\ensuremath{\mathcal{W}}}
 
 \def\cL{\ensuremath{\mathcal{L}}}

 \def\cH{\ensuremath{\mathcal{H}}}

\newcommand{\ignore}[1]{}

\def\sem{\mathit{sem}}

\def\ar{\leftarrow}
\def\rar{\rightarrow}

 \def\beq{\begin{equation}}
 \def\eeq#1{\label{#1}\end{equation}}
 \def\ba{\begin{array}}
 \def\ea{\end{array}}

\newcommand{\optdll}{{\sc opt-dll}\xspace}
\newcommand{\minone}{{\sc min-one}\xspace}
\newcommand{\minoneasp}{{\sc min-one-asp}\xspace}

\newcommand{\minonesub}{{\sc min-one$_\subseteq$}\xspace}
\newcommand{\distancesat}{{\sc distancesat}\xspace}

\newcommand{\distancesatsub}{{\sc distancesat$_\subseteq$}\xspace}

\def \less#1#2{{#1}[{\setminus{#2}}]}

\begin{document}

\lefttitle{Yuliya Lierler}

\title[An Abstract View on Optimizations in Propositional Frameworks]{
An Abstract View on Optimizations in Propositional Frameworks}

  \begin{authgrp}
\author{\gn{Yuliya Lierler}}
\affiliation{University of Nebraska Omaha}
\end{authgrp}
\jnlPage{\pageref{firstpage}}{\pageref{lastpage}}
\jnlDoiYr{2022}
\doival{10.1017/xxxxx}

\maketitle

\label{firstpage}









\begin{abstract}
Search-optimization problems are plentiful in scientific and engineering domains. Artificial intelligence has long contributed to the development of search algorithms and declarative programming languages geared toward solving and modeling search-optimization problems.
Automated reasoning and knowledge representation are the subfields of AI that are particularly vested in these developments.
Many popular automated reasoning paradigms provide users with languages supporting optimization statements:     answer set programming or MaxSAT on \minone, to name a few.
These paradigms vary significantly in their languages and in the ways they express quality conditions on computed solutions. Here we propose a unifying framework of so-called weight systems that eliminates syntactic distinctions between paradigms and allows us to see essential similarities and differences between optimization statements provided by paradigms. This unifying outlook has significant 
	simplifying and explanatory potential in the studies of optimization and modularity in automated reasoning and knowledge representation. It also supplies researchers with a convenient tool for proving the formal properties of distinct frameworks; bridging these frameworks; and facilitating the development of translational solvers.
\end{abstract}
\section{Introduction}
Artificial intelligence is a powerhouse for delivering algorithmic frameworks to support solutions to
search-optimization problems that are plentiful in scientific and engineering domains. 
Automated reasoning and knowledge representation are the subfields of AI that are particularly vested in developing general-purpose search algorithms and declarative programming languages specifically geared towards formulating constraints of search-optimization problems.
Various automated reasoning paradigms provide users with languages supporting optimization statements. Indeed, consider  such popular AI paradigms as
 answer set programming with weak constraints~(ASP-WC)~\citeps{alv18} and propositional satisfiability with optimizations~(MaxSAT family)~\citeps{rob10}. 
In practice, when search problems are formulated there is {\em often} an interest not only in   identifying a solution, but also in pointing at the one that is optimal with respect to some criteria. Another way to perceive this setting is by having interplay of ``hard'' and ``soft'' modules (drawing a parallel to terminology  used in formulating partial weighted MaxSAT). Hard modules formulate immutable constraints of a problem, i.e., requirements that solutions to a problem  {\em must} satisfy in order to deserve being called a solution. Soft modules express conditions that are closer to preferences. 
Consider the well known {\em travelling salesman problem} (TSP) that asks the following: ``Given a list of cities and the information on how the cities are connected by roads with associated lengths, what are the possible routes that visit each city and return to the origin city?  What are the shortest routes?" Here, the first question of TSP formulates ``hard'' requirements. In order for a sequence of cities to form a solution it has to form a special kind of path. The second question of TSP formulates a preference. Out of all possible solutions we are interested in these that satisfy some numeric criterion. In ASP-WC, for example, a traditional logic program composed of rules is used to express a hard fragment  of a given problem, while 
weak constraints are used to formulate a soft fragment. 

Supporting various kinds of optimizations on an encoding and solving level is a holy grail of ASP-WC as in practical applications it is common that there are criteria which can make one solution to be preferred over the other. 
Yet, some approaches to answer set solving that rely on translations to related automated reasoning (AR) paradigms -- ``translational'' solvers such as  {\sc cmodels}~\citeps{giu06} and {\sc lp2sat}~\citeps{jan06a}, which translate logic programs into the propositional satisfiability (SAT) problem~\citeps{satprimer} --  do not provide any support for weak constraints. One reason is that SAT itself has no support for formulating soft requirements. MaxSAT and its variants are extensions to SAT supporting optimizations. The formulations of these extensions  significantly differ syntactically and semantically from those used in ASP-WC  so that the \textit{exact} link, between ASP-WC  and MaxSAT formalisms, required in implementing translational approaches is not obvious. 
In general, optimizations in different areas of AR (see, for instance,~\citeps{nie06a,chu11,and12,ros13,bre15a,alv18})  are studied in separation with no clear articulation of the {\em exact} links between the languages expressing optimization criteria and their implementations. 
Yet, it is due to note that their relation is often obvious to researchers developing solving procedures within AR. The algorithmic techniques have been borrowed from solvers in one paradigm to solvers in the other, see, for instance, the paper by~\cite{alv20} explaining how algorithmic techniques behind MaxSAT can be utilized in solving problems formulated in ASP-WC. Here, we focus on uncovering the precise semantic relation between distinct paradigms geared to support translational approaches. 

This paper takes modularity and abstraction as  key tools for building a thorough understanding between related and, yet, disperse advances pertaining to optimizations or soft modules within different AR communities.
\cite{lt2015} proposed an abstract modular framework that allows us to bypass the syntactic details of a particular logic and study advances in AR from a bird's eye view.  That framework is appropriate for capturing varieties of logics within hard modules.
We extend the framework in a way that soft modules can be formulated and studied under one roof. We illustrate how a family of SAT based optimization formalisms such as MaxSAT, weighted MaxSAT, and partial weighted MaxSAT (pw-MaxSAT) can be embedded into the proposed framework. We also illustrate how ASP-WC  fits into the same framework. 
The predominant focus of this work are so called quantitative optimizations/preferences. An objective function for assessing quality of computed models is of central importance in these methods. Weighted-conditions of introduced weighted systems are vehicles for expressing these objective functions. Qualitative preferences is an alternative. These maybe expressed by means of preference relations on the level of atoms (literals - atoms or their complements) or interpretations. We illustrate how some optimization frameworks stemming from qualitative preferences within SAT such as \minone~\citeps{cre01} and \distancesat~\citeps{bai06} are conveniently captured by w-system. This fact allows us an immediate extension of these ideas to the case of answer set programming.

In addition, we carefully study the introduced abstract framework and illustrate a number of its formal properties. These properties immediately translate into formal claims about specific instances of this framework such as MaxSAT or ASP-WC. The paper culminates in a result illustrating how ASP-WC programs can be processed by means of MaxSAT/MinSAT solvers. The opposite link also becomes apparent.
To summarize, we propose to utilize abstract view on logics and modularity as tools for constructing overarching view for distinct criteria used for optimization within different AR communities. 
We would like to point at work by \cite{alv18b}, where the authors also realized the importance of abstracting away the syntactic details of the formalisms to describing hard and soft constraints of a considered problem in order to streamline the utilization of existing efficient solving techniques in new settings. 
Similarly, the work by~\cite{ens19} is devoted to preferences in declarative programming solving.

\smallskip
\noindent
{\em Paper outline~~}
We start the paper by reviewing the concepts of an abstract logic and modular systems. We then introduce a notion of weighted abstract modular systems -- key artifact of this paper.  In Section~\ref{sec:instances}, we show how these weighted systems  naturally capture MaxSAT family and answer set programming with weak constraints.
Section~\ref{sec:qual} is devoted to the case of qualitative preferences and we illustrate  how these weighted systems  capture such SAT-based formalisms as \minone and \distancesat, which also lend themselves to newly defined frameworks of, for instance, \minone for the case of answer set programming.
Section~\ref{sec:formalProperties} enumerates formal properties of w-systems showcasing how these properties immediately translate into properties of MaxSAT family and ASP-WC paradigms. In conclusion of this section we are able to provide a precise link between pw-MaxSAT/pw-MinSAT and ASP-WC. The concluding section of the paper lists proofs.

Parts of this paper appeared in the proceedings of the 17th Edition of the European Conference on Logics in Artificial Intelligence \citeps{lie21}. In particular, this paper unlike its conference predecessor provides a complete account of proofs for the stated formal results. It also extends the use of weighted systems to the case of qualitative preferences utilized in such formalisms as \minone and \distancesat.

\section{Review: Abstract Logics and Modular Systems}
\label{mbms}
We start with the review 
of an abstract logic by \cite{bre07}. We illustrate how this abstract concept captures SAT and logic programs under answer set semantics.
We then review model-based abstract modular systems advocated by \cite{lt2015}. 

A \emph{language} is a set $L$ of \emph{formulas}. A \emph{theory} is 
a subset of~$L$. Thus the set of theories is closed under union and
has the least and the greatest elements: $\emptyset$ and $L$.
We call a theory a {\em singleton} if it is an element/a formula in $L$ (or a singleton subset, in other words). This definition ignores any syntactic details behind the concepts of
 a formula and a theory.
A \emph{vocabulary} is possibly an infinite countable set of \emph{atoms}.
Subsets of a vocabulary $\sigma$ represent (classical propositional)
\emph{interpretations} of $\sigma$. We write $\Int(\sigma)$ for the 
family of all interpretations of a vocabulary~$\sigma$. 
It is customary for a given vocabulary $\sigma$, to identify a set~$X$ of atoms over $\sigma$ with~an assignment function  that assigns the truth value {\em true} to every atom in~$X$ and~{\em false} to every atom in $\sigma\setminus X$. In classical propositional logic, interpretations are often defined by means of such assignment functions.  

%

\begin{Definition}
A \emph{logic} is a triple $\cL=(L_\cL,\sigma_\cL,\sem_\cL)$, where
\begin{enumerate}
\item $L_\cL$ is a language (\emph{language} of $\cL$)
\item $\sigma_\cL$ is a vocabulary (\emph{vocabulary} of $\cL$)
\item $\sem_\cL:2^{L_\cL} \rightarrow 2^{\Int(\sigma_\cL)}$ is 
a function from theories in $L_\cL$ to
collections of interpretations 
(\emph{semantics} of~$\cL$)
\end{enumerate}
\end{Definition}
If a logic $\cL$ is clear from the context, we omit the subscript~$\cL$ 
from the notation of the language, the vocabulary and the semantics of 
the logic. 

\cite{bre07} showed that this abstract notion 
of a logic captures default logic, propositional logic, and logic programs 
under the answer set semantics. 
For example, the logic  $\cL=(L,\sigma,\sem)$,
where
\begin{enumerate}
\item $L$ is the set of propositional formulas over $\sigma$,
\item $\sem(F)$, for a theory $F\subseteq L$, is the set of propositional
models of theory $F$ (where we understand an interpretation to be a model of  theory $F$ if it is a model of each element/propositional formula in $F$)  over $\sigma$,
\end{enumerate}
captures propositional logic. (We assume the readers familiarity with the key concepts of propositional logic: formulas, interpretations, models.)
We call this logic  a {\em pl-logic}.
A {\em clause} is a propositional formula of the form 
\beq
 \neg a_1\vee\dotsc\vee \neg a_\ell\vee\ a_{\ell+1}\vee\dotsc\vee\   a_m,
\eeq{eq:wcwcc}
where every $a_i$ is an atom.
If we restrict elements of $L$ to be  clauses, then we call $\cL$  a {\em sat-logic}. 
Intuitively, the finite theories in sat-logic can be identified with CNF formulas.
Say, sat-logic theory 
$\{(a\vee b),(\neg a\vee \neg b)\}$
stands for the formula 
\beq
(a\vee b)\wedge (\neg a\vee \neg b).
\eeq{eq:abtheory}

We now review logic programs. 
A \emph{logic program} over $\sigma$ is a 
finite 
set of \emph{rules} of the form
\begin{equation}\label{e:rule}
\begin{array}{l}
a_0\ar a_1,\dotsc, a_\ell,\ not\  a_{\ell+1},\dotsc,\ not\  a_m,\ 
\end{array}
\end{equation}
where $a_0$ is an atom in $\sigma$ or $\bot$ (empty), and each $a_i$ $(1\leq i\leq m)$ 
is an atom in $\sigma$.
The expression $a_0$ is the \emph{head} of the rule. 
The expression on the right hand side of the arrow, is the \emph{body}. When the body is empty we drop  $\ar$. Expression $not\  a_{\ell+1},\dotsc,\ not\  a_m$ is called {\em negative part of the body}.

We say that a set~$X$ of atoms {\em satisfies} rule~\eqref{e:rule}, if~$X$ satisfies the propositional formula
$$
a_1\wedge\dotsc\wedge a_\ell\wedge\ \neg  a_{\ell+1}\wedge\dotsc\wedge\ \neg  a_m\rar a_0.\ 
$$
The {\sl reduct}~$\Pi^X$ of a program~$\Pi$ relative to a set~$X$ of atoms is 
obtained by first removing all rules~\eqref{e:rule} such that~$X$ does not satisfy the propositional formula corresponding to the negative part of the body
$\neg  a_{\ell+1} \wedge\ldots\wedge\neg  a_m ,$
and replacing all remaining rules with their counterparts without negative part of the body
$$a\ar a_1,\ldots, a_\ell.$$ 
For example, let $\Pi_1$ denote a program
\beq
\ba{l}
a\ar not\ b\\
b\ar not\ a.
\ea
\eeq{ex:slp}
The reduct  ~$\Pi_1^{\{a\}}$  consists of a single rule
$$
a.
$$

A set~$X$ of atoms is an {\em answer set}, if it is the minimal set that satisfies all rules of~$\Pi^X$~\citeps{lif99d}. 
For instance, program~\eqref{ex:slp}
has two answer sets $\{a\}$ and $\{b\}$.

Abstract logics of Brewka and Eiter subsume the formalism of
logic programs under the  answer set semantics. 
Indeed, let us consider
a logic  $\cL=(L,\sigma,\sem)$, where
\begin{enumerate}
\item $L$ is the set of logic program rules over $\sigma$,
\item $\sem(\Pi)$, for a program $\Pi\subseteq L$, is the set of  answer sets of $\Pi$ over $\sigma$,
\end{enumerate}
We call this logic the {\em lp-logic}.

\cite{lt2015} propose (model-based) abstract modular systems that allow us to construct heterogeneous systems based of ``modules'' stemming from a variety of logics. 
We now  review their framework.

\begin{Definition}
Let $\cL=(L_\cL,\sigma_\cL,\sem_\cL)$ be a logic.
 A theory of $\cL$, that is, a subset of the 
language $L_\cL$ is called a \emph{(model-based) $\cL$-module} 
(or a \emph{module}, if the explicit 
reference to its logic is not necessary). An interpretation $I\in
\Int(\sigma_\cL)$ is a \emph{model} of an $\cL$-module/theory $T$
if $I\in\sem_\cL(T)$. 
\end{Definition}


We use words theory and modules interchangeably at times. Furthermore, for a theory/module in pl- or sat-logics we often refer to those as propositional or SAT formulas (sets of clauses).  For a theory/module in lp-logic we  refer to it as a logic program.

For an interpretation
$I$, by $I_{|\sigma}$ we denote an
interpretation over vocabulary 
$\sigma$ constructed from~$I$ by dropping all its members not in
$\sigma$. For example, let $\sigma_1$ be a vocabulary such that $a\in\sigma_1$ and $b\not\in\sigma_1$, then $\{a,b\}_{|\sigma_1}=\{a\}$. 
We now extend the notion of a model to vocabularies that go beyond the one of a considered module in a straightforward manner.  For an $\cL$-module $T$ and an interpretation $I$ whose vocabulary is a superset of the vocabulary $\sigma_\cL$ of $T$, we say that  
$I$ is a {\em model} of $T$, denoted $I\models T$, if $I_{|\sigma_\cL}\in\sem_\cL(T)$.
This extension is in spirit of a convention used in classical logic (for example, given a propositional formula $p\wedge q$ over vocabulary $\{p,q\}$ we can speak of interpretation assigning true to propositional variables $\{p,q,r\}$ as a model to this formula).

\begin{Definition}
A set of modules, possibly in different logics and over different
vocabularies is a \emph{(model-based) abstract modular system (AMS)}. 
For an abstract modular system $\cH$, the union of the 
vocabularies of the logics of the modules in $\cH$ forms the \emph{vocabulary} 
of $\cH$, denoted by $\sigma_\cH$. An interpretation  $I\in \Int(\sigma_\cH)$
is a \emph{model} of $\cH$ when for every module $T\in\cH$,  $I$ is a model of $T$. (It is easy to see that we can extend the notion of a model to interpretations whose vocabulary goes beyond $\sigma_\cH$ in a straightforward manner.)
\end{Definition}
When an AMS consists of a single module $\{F\}$ we  identify it with module $F$ itself.

\section{Weighted Abstract Modular Systems}\label{sec:wams}
In practice, we are frequently interested not only in identifying models of a given logical formulation of a problem (hard fragment) but identifying models that are deemed optimal according to some criteria (soft fragment). Frequently, multi-level optimizations are of interest. 
An AMS framework is geared towards capturing heterogeneous solutions for formulating hard constraints. Here we extend it to enable the formulation of soft constraints. 
We start by introducing a ``w-condition'' -- a module  accommodating notions of a level and a weight. We then introduce w-systems -- a generalization of AMS that accommodates new kinds of modules. In conclusion, we  embed multiple popular automated reasoning optimization formalisms into this framework.

\newpage
\begin{Definition}
Let $\cL=(L_\cL,\sigma_\cL,\sem_\cL)$ be a logic.
 A pair $(T_\cL,w@l)$ -- 
  consisting of a theory $T_\cL$ of logic~$\cL$ (or $\cL$-module) and an expression $w@l$, where  $w$ is an integer and $l$ is a positive integer -- is called an \emph{ $\cL$-w(eighted)-condition}  
(or a \emph{w-condition}, if the explicit 
reference to its logic is not necessary). 
We refer to integers $l$ and $w$ as {\em levels} and {\em weights}, respectively.
An interpretation $I\in
\Int(\sigma_\cL)$ is a \emph{model} of an~$\cL$-w-condition $B=(T_\cL,w@l)$, denoted $I\models B$
if $I\in\sem_\cL(T_\cL)$. 
A mapping $\br{I\models B}$ is defined as follows
\beq \br{I\models B}=\begin{cases}
  w&\hbox{when $I\models B$, }  \\
  0&\hbox{otherwise.}  
\end{cases}
\eeq{eq:isat}
By $\level{B}$ and $\weight{B}$ we denote level $l$ and weight $w$ associated with w-condition $B$, respectively.
\end{Definition}
We identify w-conditions of the form $(T,w@1)$ with expressions $(T,w)$ (i.e., when the level is missing it is considered to be one). 

For a collection $\cS$ of w-conditions, the union of the 
vocabularies of the logics of the w-conditions  in $\cS$ forms the \emph{vocabulary} 
of $\cS$, denoted by $\sigma_\cS$. 

\vspace{3em}

\begin{Definition}
A pair $(\cH,\cS)$ consisting of an AMS $\cH$ and a set $\cS$ of w-conditions  (possibly in different logics and over different
vocabularies) so that $\sigma_\cS\subseteq\sigma_\cH$ is called a \emph{ w(eighted)-abstract modular system} (or {\em w-system}). 
\end{Definition}
Let $\cW=(\cH,\cS)$ be
a w-system ($\cH$ and $\cS$ intuitively stand for {\em hard} and {\em soft}). 
The 
vocabulary of $\cH$ forms the \emph{vocabulary} 
of $\cW$, denoted by $\sigma_\cW$. We now state the definition of a model of a system that is solely based on the semantics of the hard part.
\begin{Definition}\label{def:modelwsys}
Let pair $\cW=(\cH,\cS)$ be a w-system.
An interpretation $I\in \Int(\sigma_\cW)$
is a \emph{model} of~$\cW$ if it is a model of  AMS 
$\cH$.
\end{Definition}

For the special case when all w-conditions of a considered w-system have the form~$(T,w)$, we  provide a simple  definition of an optimal/min-optimal model that takes soft part of w-system into account so that models are distinguished based on their ``w-conditions quality''. This definition is inspired by the way the solutions to partially weighted MaxSAT problems are stated (see Section~\ref{sec:maxsat}).

\begin{Definition}\label{def:optimalmodelwsysSimple}
Let pair $\cW=(\cH,\cS)$ be a w-system, where $\cS$ has the form $\{(T_1,w_1),\dots,(T_n,w_n)\}$.
A model $I^*$ of~$\cW$ is {\em optimal} if 
$I^*$ satisfies equation
$$
I^*=\displaystyle{ \argmax_{I} {\sum_{B\in\cS}{ \br{I\models B}}}
},
$$ 
where $I$ ranges over models of~$\cW$.
A model $I^*$ of~$\cW$ is {\em min-optimal} if it satisfies the conditions of optimal model, where in the equation above  we replace $max$ by $min$.
\end{Definition}

Let $\cW=(\cH,\cS)$ be w-system.
For a level~$l$, by $\cW_l$ we denote the subset of $\cS$ that includes all w-conditions whose level is~$l$.
By $\level{\cW}$ we denote the set of all levels associated with w-system~$\cW$ constructed as 
$
\{\level{B}\mid B\in \cS\}$.

We now provide the definition of an optimal model for the general form of w-systems. It is inspired by a definition of an optimal answer set coming from the ASP-WC framework, where multi-level optimizations are standard (see Section~\ref{sec:olp}). 
\begin{Definition}\label{def:optimalmodelwsysASP}
Let $I$ and $I'$ be models of w-system $\cW=(\cH,\cS)$.
Model $I'$ {\em min-dominates}  $I$ if 
there exists a level $l\in\level{\cW}$ such that
following conditions are satisfied:
\begin{enumerate}
\item for any level $l'>l$  the following equality holds
$$
\displaystyle{ \sum_{B\in\cW_{l'}}{ \br{I\models B}}}
=
\displaystyle{ \sum_{B\in\cW_{l'}}{ \br{I'\models B}}}
$$
\item the following inequality holds for level $l$
$$
\displaystyle{ \sum_{B\in\cW_l}{ \br{I'\models B}}}
<
\displaystyle{ \sum_{B\in\cW_l}{ \br{I\models B}}}
$$ 
\end{enumerate}
Model $I'$ {\em max-dominates}  $I$ if 
we change less-than symbol by greater-than symbol in the inequality of Condition~\ref{l:cond2}.

A model $I^*$ of $\cW$ is {\em optimal}  if there is no model~$I'$ of $\cW$ that max-dominates $I^*$. 
A model $I^*$ of~$\cW$ is {\em min-optimal}  if there is no model~$I'$ of $\cW$ that min-dominates $I^*$. 
\end{Definition}

We now attempt to reconcile terminology used in Definitions~\ref{def:optimalmodelwsysSimple} and~\ref{def:optimalmodelwsysASP}. We state an alternative definition to optimal and min-optimal models of Definition~\ref{def:optimalmodelwsysASP} and then provide a formal result on their equivalence.
Let $\cW=(\cH,\cS)$ be w-system.
For a level $l\in \level{\cW}$ by $\prev{l}$ we denote the least level in $\level{\cW}$ that is greater than $l$ (it is obvious that for the greatest level in $\level{\cW}$,  $\prev{l}$ is undefined). For example, for levels in~$\{2, 6, 8, 9\}$,
$\prev{2}=6$, $\prev{6}=8$, and~$\prev{8}=9$.
\begin{Definition}\label{def:optimalmodelwsys}
Let pair $\cW=(\cH,\cS)$ be a w-system.
For  level $l\in\level{\cW}$, a model $I^*$ of $\cW$ is {\em $l$-optimal} if 
$I^*$ satisfies equation
\beq
I^*=\displaystyle{ \argmax_{I} {\sum_{B\in\cW_l}{ \br{I\models B}}}
},
\eeq{eq:condeqlmin}
where
\begin{itemize}
    \item $I$ ranges over models of $\cW$ if $l$ is the greatest level in $\level{\cW}$, 
\item $I$ ranges over $\prev{l}$-optimal models of $\cW$, otherwise.
\end{itemize}
We call a model  {\em $l$-min-optimal} if $max$ is replaced by $min$ in~\eqref{eq:condeqlmin}. 

A model $I^*$ of $\cW$ is {\em optimal} if $I^*$ is $l$-optimal model for every level $l\in\level{\cW}$.
A model $I^*$ of~$\cW$ is {\em min-optimal} if $I^*$ is $l$-min-optimal model for every level $l\in\level{\cW}$.
\end{Definition}
\begin{Proposition}\label{prop:eqdefs}
Definitions~\ref{def:optimalmodelwsysASP} and~\ref{def:optimalmodelwsys} are equivalent.
\end{Proposition}
The proofs of the formal results of this paper are presented in its final section. The proof of the claim of Proposition~\ref{prop:eqdefs} relies on the following lemmas that we list here as they refer to interesting  properties of $l$-optimal models.
\begin{Lemma}\label{lem:loptimal}
For a w-system $\cW=(\cH,\cS)$ and levels $l,l' \in \level{\cW}$ so that $l<l'$, any $l$-optimal/$l$-min-optimal model of $\cW$ is also an $l'$-optimal/$l'$-min-optimal model of $\cW$ as well as model of $\cW$. For the least level $l$ in $\level{\cW}$,   $l$-optimal/$l$-min-optimal model of $\cW$ is an optimal/min-optimal model of~$\cW$.
\end{Lemma}
\begin{Lemma}\label{lem:loptimal2}
For a w-system $\cW=(\cH,\cS)$ and levels $l,l' \in \level{\cW}$ so that $l\leq l'$, and any $l$-optimal/$l$-min-optimal models $I$ and $I'$ of $\cW$,
$$\sum_{B\in\cW_{l'}}{ \br{I\models B}}= \sum_{B\in\cW_{l'}}{ \br{I'\models B}}.$$
\end{Lemma}
Section~\ref{sec:formalProperties}  presents many interesting formal properties of  w-systems. Yet, before proceeding towards that presentation in the following section we illustrate several AI formalisms that can be viewed as instances of w-systems. With that our later discussion of formal properties of w-systems immediately applies to these AI formalisms.

\section{Instances of W-Systems}\label{sec:instances}
\subsection{MaxSAT/MinSAT Family as W-systems}\label{sec:maxsat}
We now restate the definitions of {\em MaxSAT}, {\em weighted MaxSAT} and {\em pw-MaxSAT}~\citeps{rob10}. We then show how  these formalisms are captured in terms of w-systems. In the sections that follow we use w-systems to model logic programs with  weak constraints. The uniform language of w-systems allows us to prove properties of theories in these different logics by eliminating the reference to their syntactic form. Later in the paper we provide translation from logic programs with weak constraints to pw-MaxSAT problems.

To begin we introduce a notion of so called $\sigma$-theory. For a vocabulary $\sigma$ and a logic $\cL$ over this vocabulary ($\sigma_\cL=\sigma$), we call theory $T_\cL$ a
{\em $\sigma$-theory/$\sigma$-module}  when it satisfies property $sem(T_\cL)=Int(\sigma)$.
For example, in case of pl-logic or sat-logic a conjunction of clauses of the form $a\vee \neg a$ for every atom $a\in\sigma$ forms a $\sigma$-theory. For a $\sigma$-theory a logic of the theory becomes immaterial so we allow ourselves to denote an arbitrary $\sigma$-theory by
 $T_\sigma$ disregarding the reference to its logic.

As customary in propositional logic given an interpretation $I$  and a propositional formula $F$, we write $I\models F$ when~$I$ satisfies $F$ (i.e., $I$ is a model of $F$). 
A mapping $\br{I\models F}$ is defined as 
\beq \br{I\models F}=\begin{cases}
  1&\hbox{when $I\models F$, }  \\
  0&\hbox{otherwise.}  
\end{cases}
\eeq{eq:isatorig}
\begin{Definition}
An interpretation $I^*$ over vocabulary $\sigma$ is a {\em solution} to {\em MaxSAT problem} $F$, where $F$ is a CNF formula over~$\sigma$, when it satisfies the equation 
$$
I^*=arg \max_{I}{\sum_{C\in F}{\br{I\models C}}},
$$
where $I$ ranges over all interpretations over $\sigma$.
\end{Definition}
The following result illustrates how w-systems can be used to capture MaxSAT problem.
\begin{Proposition}\label{prop:maxsat}
Let $F$ be a MaxSAT problem over $\sigma$.
The optimal models of  w-system $(T_\sigma,\{(C,{1})\mid C\in F\})$ -- 
where pairs of the form $(C,{1})$ are  sat-logic w-conditions --  form the set of solutions for~$F$. 
\end{Proposition}
Proposition~\ref{prop:maxsat}  allows us to identify w-systems of particular form with MaxSAT  problems. For example, any w-system of the form $\big(T_\sigma,\{(C_1,1),\dots(C_n,1)\}\big)$ -- where $C_i$ $(1\leq i\leq n)$ is a singleton sat-logic theory -- can be seen as a MaxSAT problem composed of clauses $\{C_1,\dots,C_n\}$.

\begin{Definition} 
 A {\em weighted MaxSAT problem}~\citeps{arg08}/{\em weighted MinSAT problem}~\citeps{chu11} is defined as a set $(C,w)$ of pairs, where $C$ is a clause and $w$ is a positive integer.
An interpretation~$I^*$ over vocabulary $\sigma$ is a  {\em solution} to weighted MaxSAT problem $P$ over~$\sigma$, when it satisfies the equation 
\beq
I^*=arg \max_{I}{\sum_{(C,w)\in P}{w\cdot \br{I\models C}}},
\eeq{eq:wmaxsat}
where $I$ ranges over all interpretations over $\sigma$.
An interpretation $I^*$ over vocabulary $\sigma$ is a  {\em solution} to weighted MinSAT problem $P$ over $\sigma$, when it satisfies the equation~\eqref{eq:wmaxsat}, where we replace $max$ with $min$. 
\end{Definition}

The following proposition allows us to identify w-systems of particular form with weighted MaxSAT/MinSAT problems.

\bigskip
\begin{Proposition}\label{prop:wmaxsat}
Let $P$ be a weighted MaxSAT problem over~$\sigma$.
The optimal  models of w-system $(T_\sigma,P)$ --
where each element in $P$ is understood as a sat-logic w-condition --
 form the set of solutions for~$P$. 

Let $P$ be a weighted MinSAT problem over~$\sigma$.
The min-optimal  models of w-system $(T_\sigma,P)$ --
where each element in $P$ is understood as a sat-logic w-condition --
 form the set of solutions for~$P$. 
\end{Proposition}


\begin{Definition}
A {\em partially-weighted MaxSAT problem or pw-MaxSAT problem}~\citeps{fu06} ({\em pw-MinSAT problem}~\citeps{chu11})
is defined as a pair $(F,P)$ over vocabulary $\sigma$, where~$F$ is a CNF formula over $\sigma$ and~$P$ is  a weighted MaxSAT problem (a weighted MinSAT problem) over~$\sigma$. Formula~$F$ is referred to as {\em hard} problem fragment, whereas clauses in $P$ form {\em soft} problem fragment. An interpretation $I$ over~$\sigma$ is a {\em model} of pw-MaxSAT problem/pw-MinSAT problem $(F,P)$, when 
$I$ is a model of $F$. 

Let $(F,P)$ be a pw-MaxSAT problem  over vocabulary $\sigma$.
A model $I^*$ of $(F,P)$ is optimal when
it satisfies  equation~\eqref{eq:wmaxsat}, where $I$ ranges over models of $F$. 

Let $(F,P)$ be a pw-MinSAT problem  over vocabulary $\sigma$.
A model $I^*$ of $(F,P)$ is optimal when
it satisfies  equation~\eqref{eq:wmaxsat}, where $I$ ranges over models of $F$ and $max$ is replaced by $min$. 
\end{Definition}

The following proposition allows us to identify w-systems of particular form with pw-MaxSAT/pw-MinSAT problems.
\begin{Proposition}\label{prop:pwmaxsat}
Let $(F,P)$ be a pw-MaxSAT problem  over vocabulary $\sigma$.
The models and optimal models of w-system $(F,P)$ --
where $F$ is a sat-logic module and each element in $P$ is understood as a  sat-logic w-condition --
 coincide with the models and optimal models of pw-MaxSAT problem $(F,P)$, respectively.
 
Let $(F,P)$ be a pw-MinSAT problem  over vocabulary $\sigma$.
The models and min-optimal models of w-system $(F,P)$ --
where $F$ is a sat-logic module and each element in $P$ is understood as a  sat-logic w-condition --
 coincide with the models and optimal models of pw-MinSAT problem $(F,P)$, respectively.
 \end{Proposition}
 We now present sample pw-MaxSAT problem to illustrate some definitions at work. 
Take $F_1$ to denote  sat-theory module~\eqref{eq:abtheory}.
Consider a pw-MaxSAT problem 
\beq
(F_1,\{(a,1),(b,1),(a\vee\neg b,2), (\neg a\vee b,0)\}).
\eeq{eq:partmsat}
The models of this problem are $\{a\}$ and $\{b\}$; and its optimal model is $\{a\}$. 
If we consider this pair as  pw-MinSAT problem then $\{b\}$ is an optimal model of this problem. 
Consider a pw-MinSAT problem 
\beq
(F_1,\{(\neg a\vee b,2)\}).
\eeq{eq:pwminsat}
The models of this problem are $\{a\}$ and $\{b\}$; and its optimal model is $\{a\}$.

Embedding family of MaxSAT/MinSAT problems into w-systems realm provides us with immediate means to generalize their definitions to accommodate
\begin{itemize}
    \item 
    negative weights; 
    \item 
    distinct levels accompanying weight requirement on its clauses; 
    \item  
    removing restriction from its basic syntactic object being a clause and allowing, for example,  arbitrary propositional formulas, as a logic for its module and w-conditions.
\end{itemize}
Consider the following definition of MinMaxPL Problem, meant to be a counterpart of pw-MaxSAT/pw-MinSAT defined for arbitrary propositional formulas and incorporating enumerated items. We call a w-system~$(F,S)$ a {\em MinMaxPL problem}, when $F$ is a pl-logic module and each w-condition in $S$ is in pl-logic. 
It is easy to see that any pw-MaxSAT problem is a special case instance of MinMaxPL problem, whose optimal models coincide. 
Any pw-MinSAT problem is a special case instance of MinMaxPL problem, where optimal models of pw-MinSAT and 
min-optimal models of respective MinMaxPL problem coincide. 

The pair 
\beq
(F_1,\{(a,1),(b,1@3),(a\vee\neg b,2), (\neg a\vee b,0)\})
\eeq{eq:maxpl}
forms a sample MinMaxPL problem that differs from~\eqref{eq:partmsat} in boosting the level of one of its w-conditions. The optimal model of this system is $\{b\}$.
In the sequel we illustrate that presence of levels and negative weights in w-systems can often be considered as syntactic sugar. Also, the concept of min-optimal model can  be expressed in terms of optimal models of a closely related w-system. Yet, from the perspective of knowledge representation, convenience of modeling, algorithm design for search procedures such features are certainly of interest and deserve an attention and thorough understanding.

\subsection{Optimizations in Logic Programming or Answer Set Programming with Weak Constraints}\label{sec:olp}  
  We now review 
 a definition of a logic program with weak constraints following the lines of~\citeps{cal15}.
        A {\em weak constraint} has the form
\beq 
\wr a_1,\dotsc, a_\ell,\ not\  a_{\ell+1},\dotsc,\ not\  a_m[w@l],
\eeq{eq:wc}
where $m>0$ and $a_1,\ldots,a_m$ are atoms,  $w$ (weight) is an integer, and  $l$ (level) is a positive integer. 
In the sequel, we abbreviate expression
\beq
\wr  a_1,\dotsc, a_\ell,\ not\  a_{\ell+1},\dotsc,\ not\  a_m
\eeq{eq:wcbody}
occurring in~\eqref{eq:wc} as $D$ and identify it with the propositional formula
\beq a_1\wedge\dotsc\wedge a_\ell\wedge\ \neg  a_{\ell+1}\wedge\dotsc\wedge\ \neg  a_m.
\eeq{eq:wcbodyf}
An {\em optimization program} (or {\em o-program})  over vocabulary $\sigma$ is a pair~$(\Pi,W)$, where $\Pi$ is a logic program over  $\sigma$ and~$W$ is a finite set of weak constraints over $\sigma$.

Let $\cP=(\Pi,W)$ be an optimization program over vocabulary $\sigma$ (intuitively, $\Pi$ and~$W$ form  {\em hard} and {\em soft} fragments, respectively). 
By $\level{\cP}$ we denote the set of all levels associated with optimization program $\cP$ constructed as 
$
\{l\mid\, D[w@l]\in W\}$.
\begin{Definition}
Let $\cP=(\Pi,W)$ be an optimization program over  $\sigma$. 
Set~$X$ of atoms over $\sigma$ is an {\em answer set} of~$\cP$, when it is an answer set of $\Pi$.

Let $X$ and $X'$ be answer sets of $\cP$.
Answer set $X'$ {\em dominates}  $X$ if 
there exists a level $l\in\level{\cP}$ such that
following conditions are satisfied:
\begin{enumerate}
\item\label{l:cond1} for any level $l'$ that is greater than $l$  the following equality holds
$$
\displaystyle{ \sum_{D[w@l']\in W}{w \cdot\br{X\models D}}
} =\displaystyle{\sum_{D[w@l']\in W}{w \cdot\br{X'\models D}}
}
$$
\item\label{l:cond2} the following inequality holds for level $l$
$$
\displaystyle{\sum_{D[w@l]\in W}{w \cdot\br{X'\models D}}
}<\displaystyle{ \sum_{D[w@l]\in W}{w \cdot\br{X\models D}}
} 
$$ 
\end{enumerate}
An answer set $X^*$ of $\cP$ is {\em optimal}  if there is no answer set~$X'$ of $\cP$ that dominates $X^*$. 
\end{Definition}


Consider a logic whose  language  is a strict subset of that of propositional logic: a language that allows only for formulas of the form~\eqref{eq:wcbodyf}, whereas its semantics is that of propositional logic. We call this logic a {\em wc-logic}.

\begin{Proposition}\label{prop:op}
Let $(\Pi,W)$ be an optimization logic program over vocabulary $\sigma$.
The models and min-optimal models of w-system $\big(\Pi,\{(D,{w@l})\mid D[w@l]\in W\}\big)$
-- where~$\Pi$ is an lp-logic module and pairs of the form $(D,{w@l})$ are  wc-logic w-conditions -- coincide with the answer sets and optimal answer sets of $(tit\Pi,W)$, respectively.
\end{Proposition}
Propositions~\ref{prop:maxsat},~\ref{prop:wmaxsat},~\ref{prop:pwmaxsat}, and~\ref{prop:op} allow us to identify MaxSAT, weighted MaxSAT, pw-MaxSAT, and o-programs with respective w-systems. In the following, we often use the terminology stemming from w-systems, when we talk of these distinct frameworks. 
For instance, we allow ourselves to identify a weak constraint~\eqref{eq:wc} with a wc-logic w-condition 
\beq(a_1\wedge\dotsc\wedge a_\ell\wedge\ \neg  a_{\ell+1}\wedge\dotsc\wedge\ \neg  a_m,w@l).\eeq{eq:wcwc}

We now exemplify the definition of an optimization program. Let $\Pi_1$ be logic program~\eqref{ex:slp}.
An optimal answer set of optimization program
\beq
(\Pi_1,\{\wr a,not\ b. [-2@1]\})
\eeq{eq:sampleop}
is $\{a\}$. We note that the answer sets and the optimal answer set   of~\eqref{eq:sampleop} coincide with the models and the optimal model of pw-MinSAT problem~\eqref{eq:pwminsat}. The formal results of this paper will show that this is not by chance.

It is worth noting that an alternative syntax is frequently used by answer set programming practitioners when they express optimization criteria:
$$
\#minimize\{w_1@l_1:lit_1,\dots,w_n@l_n:lit_n \},
$$
where $lit_i$ is either an atom $a_i$ or an expression $not\ a_i$.
This statement stands for $n$ weak constraints
$$
\ba{l}
\wr lit_1[w_1@l_1].\\
~~\cdots~~\\
\wr lit_n[w_n@l_n].
\ea
$$
Similarly, statement
$$
\#maximize\{w_1@l_1:lit_1,\dots,w_n@l_n:lit_n \},
$$
stands for $n$ weak constraints
$$
\ba{l}
\wr lit_1[-w_1@l_1].\\~\cdots~~\\
\wr lit_n[-w_n@l_n].
\ea
$$

\section{Qualitative Optimizations with W-systems}\label{sec:qual}  
So far we paid attention to what is often called quantitative optimizations or, another way to call these, quantitative preferences. An objective function for assessing quality of computed models is of central importance in these methods. Weighted-conditions of w-systems are vehicles for expressing objective functions. The MaxSAT/MinSAT family considered earlier exemplifies formalisms implementing quantitative preferences.

Qualitative preferences is an alternative. These may be expressed by means of preference relations on the level of atoms (literals - atoms or their complements) or interpretations. 
\cite{alv18b} developed an abstract framework that elegantly captures qualitative, quantitative, and mixed preferences. They introduce a concept of a knowledge base which parallels the concept of a theory/module of this work and is characterized by a set of interpretations. A knowledge base is accompanied by a preference relation over a vocabulary $\sigma$ of a considered knowledge base which is a partial order over $2^\sigma$.
\cite{alv18b} illustrate how such an abstraction is useful in stating, studying, and implementing various algorithms for computing preferred models. The focus of their work is algorithmic. Here we focus on a framework that has a ``definitional'' and ``lining'' power. We view w-systems as a convenient tool to capture formalisms defined in different terms  and then link these.
Thus, w-systems are a viable alternative to the \cite{alv18b} abstraction.  

In this section, we use w-systems to capture several formalisms studied in the literature that focus on qualitative preferences.
In the studied cases, the language of w-systems focused on quantitative preference is yet convenient to express what is considered qualitative preferences.
In particular, we
follow the steps by~\cite{giu06c}, who review \minone~\citeps{cre01}  and \distancesat~\citeps{bai06} problems and define their variants
\minonesub and \distancesatsub. 
Recall  the standard Davis-Logemann-Loveland ({\sc dll}) procedure~\citeps{dav62} -- a decision procedure for establishing the fact whether a given CNF formula is satisfiable. This procedure can be easily adapted to find a model of a satisfiable formula. 
\cite{giu06c} introduce a procedure \optdll that modifies {\sc dll} by enforcing an order in which it explores interpretations. Then, \optdll is shown to be an appropriate tool for finding a model to such qualitative optimizations formalisms as \minone, \distancesat and their variants
\minonesub and \distancesatsub. 

Here, we illustrate that w-systems naturally lend themselves into capturing  \minone and \distancesat problems. We then illustrate how the ideas behind \optdll translate into ways of utilizing w-systems for characterizing \minonesub and \distancesatsub problems. With that we introduce one more bridge between qualitative and quantitative preferences, where the later are used to capture the former. 
This is an alternative to an approach exemplified by \cite{giu06c}, where the authors focus   on the question how  qualitative preferences may capture quantitative ones. In their work, quantitative preferences (appearing within MaxSAT, for example) are translated into a propositional formula with auxiliary atoms capturing numeric constraints.

\subsection{\minone Family as  W-Systems}
\begin{Definition}[\minone and \minonesub problems]
For a vocabulary $\sigma$, let $\xi$  be a subset of $\sigma$ and~$F$ be a CNF formula over $\sigma$. 
By {\minone}$^\xi_F$ ({\minonesub}$^\xi_F$) we denote the set of interpretations $I$ over $\sigma$ that are models of $F$ and $I\cap\xi$ are of minimal cardinality (subset minimal).
We call members of {\minone}$^\xi_F$ ({\minonesub}$^\xi_F$) {\em solutions} to {\em \minone ({\minonesub}) problem $F$ over $\sigma$ characterized by~$\xi$}.
\end{Definition}

\begin{Proposition}\label{prop:minone}
Let $F$ be  \minone problem over $\sigma$ characterized by $\xi$. The solutions to this problem are formed by the optimal models of w-system 
\beq
(F,\{(\neg a,1)\mid a\in\xi\})
\eeq{eq:wminone}
(where its theory and w-conditions are considered to be in sat-logic).
\end{Proposition}
Construction of the w-system featured in Proposition~\ref{prop:minone} is  intuitive: it  prefers models that contain fewest number of atoms, capturing the requirement on minimal cardinality.  
The statement of Proposition~\ref{prop:minone} is in spirit of earlier propositions, so that w-systems can be used to provide an alternative definition for \minone problems. Indeed, w-system~\eqref{eq:wminone} can be identified with respective \minone problem.
The solutions to this problem are formed by the optimal models of w-system~\eqref{eq:wminone}.

The \minone formalism is an elaboration on a problem of satisfiability, where
 models of a propositional formula are distinguished with respect to a qualitative preference characterized by some vocabulary $\xi$. In a trivial way we may utilize the language of w-systems to define a similar framework for logic programs that we call \minoneasp. For a vocabulary $\sigma$, let $\xi$  be a subset of~$\sigma$ and $\Pi$ be a logic program over $\sigma$; 
we call  program~$\Pi$ a {\em \minoneasp problem characterized by~$\xi$}. The {\em solutions} to this problem are the optimal models of w-system
$$
(\Pi,\{(\neg a,1)\mid a\in\xi\}),
$$
where a theory of this w-system is in lp-logic  and w-conditions of this w-system are in sat-logic. Section~\ref{sec:opmaxsat} cites work that relates logic programs under answer set semantics and SAT. With these relations in mind it is easy to see how a tool for computing solutions to a \minone problem can be used for computing solutions to a \minoneasp problem (for example, a variant of \optdll introduced in~\cite[Theorem~5]{giu06c}).
The following result is immediate from the definitions of \minone and \minoneasp problems.
\begin{Proposition}\label{prop:minone-asp}
For a vocabulary $\sigma$, let $\xi$  be a subset of~$\sigma$; $\Pi$ be a logic program over $\sigma$; and~$F_\Pi$ be a CNF formula over $\sigma$ whose models coincide with the answer sets of $\Pi$.
 The solutions to \minone problem $F_\Pi$ characterized by $\xi$ coincide with 
the solutions to \minoneasp problem~$\Pi$ characterized by $\xi$.
\end{Proposition}
In this section we continue to review the formalisms in light of their origins as SAT problems enhanced by qualitative preferences. Yet, the narrative on the \minoneasp problem sheds light on how the ideas of these formalisms can easily be utilized within the world of answer set programming. 


We now turn our attention to the \minonesub problem.
\begin{Proposition}\label{prop:minonesub}
Let $F$ be  \minonesub problem over $\sigma$ characterized by $\xi$, 
let $a_1,a_2,\dots,a_n$ be a permutation of $\xi$.
Any optimal model of w-system 
\beq
(F,\{(\neg a_i,1@i)\mid a_i\hbox{ is an element of } a_1,a_2,\dots,a_n\}),
\eeq{eq:wminonesub}
is a  solution to this \minonesub problem.
\end{Proposition}

Construction of the w-system featured in Proposition~\ref{prop:minonesub} mimics/is inspired by the modification to the {\sc dll} algorithm by \optdll in order to obtain a behavior of the algorithm that is suitable for computing a solution to \minonesub problem. Indeed, procedure \optdll~\cite{giu06c} forces an order on assigning atoms in $\xi$ first exploring an option that a newly considered atom is assigned {\em false} (i.e., interpretations that do not contain this atom in them are considered first). As a result, a model found by \optdll is guaranteed to be subset minimal. 
Observe how levels in the definition of w-system~\eqref{eq:wminonesub} are a key instrument to mimic the behavior of \optdll and they do that in an {\em intuitive} fashion. Future Section~\ref{sec:levelsel} illustrates how levels can be flattened to a single one by means of adjusting the weights of the system. Yet, in that case the natural use of levels as illustrated here would be lost to crafting special-purpose weights.  

The claims of Propositions~\ref{prop:minone} and~\ref{prop:minonesub} differ substantially: Proposition~\ref{prop:minonesub} looses ``definitional'' power of Proposition~\ref{prop:minone}. Rather, it illustrates how we can construct a w-system whose optimal models capture {\em some} solutions of \minonesub problem. On the other hand, the next proposition allows us to see that all solutions of \minonesub problem can be characterized by w-systems in style of~\eqref{eq:wminonesub}. 

\begin{Proposition}\label{prop:minonesub2}
Let $F$ be  \minonesub problem over $\sigma$ characterized by $\xi$;
 $I$ be one of its solution; and
 $a_1,a_2,\dots,a_n$ be a permutation 
of $\xi$ that satisfies the following condition:
for any pair of atoms $a_i$ and $a_j$ so that $i<j$ if $a_j$ is in $I$ then $a_i$ is also in $I$ (in other words, this permutation orders atoms of $\xi$ that occur in $I$ prior to the atoms of $\xi$ that do not occur in $I$).
Then,
$I$ is an optimal model of w-system~\eqref{eq:wminonesub}.
\end{Proposition}

\subsection{\distancesat Family as  W-Systems}
\begin{Definition}[\distancesat and \distancesatsub problems]
For a vocabulary $\sigma$, let $\hat{I}$  be an interpretation of $\sigma$ and $F$ be a CNF formula over $\sigma$. 
By {\distancesat}$^{\hat{I}}_F$ 
({\distancesatsub}$^{\hat{I}}_F$) 
we denote the set of interpretations $I$ over $\sigma$ that are models of $F$ and 
set $(\hat{I}\setminus I)\cup(I\setminus \hat{I})$ are of minimal cardinality (subset minimal).
We call members of {\distancesat}$^{\hat{I}}_F$ ({\distancesatsub}$^{\hat{I}}_F$) {\em solutions} to {\em \distancesat (\distancesatsub) problem $F$ 
over $\sigma$ characterized by $\hat{I}$}.
\end{Definition}

We now claim similar formal results as in the case of the \minone family.
\begin{Proposition}\label{prop:distancesat}
Let $F$ be  \distancesat problem over $\sigma$ characterized by $\hat{I}$. The solutions to this problem are formed by the optimal models of w-system 
\beq
(F,\{(a,1)\mid a\in \hat{I}\}\cup \{(\neg a,1)\mid a\in \sigma\setminus \hat{I}\}).
\eeq{eq:done}
\end{Proposition}


\begin{Proposition}\label{prop:distancesatsub}
Let $F$ be  \distancesatsub problem over $\sigma$ characterized by $\hat{I}$, 
let $a_1,a_2,\dots,a_n$ be a permutation of $\sigma$.
Any optimal model of w-system 
\beq
\ba{rll}
(F,&\{(a_i,1@i)&\mid a_i\hbox{ is an element of } a_1,a_2,\dots,a_n \hbox{ and } a_i\in \hat{I}\}\cup\\
   &\{(\neg a_i,1@i)&\mid a_i\hbox{ is an element of } a_1,a_2,\dots,a_n \hbox{ and } a_i\in \sigma\setminus \hat{I}\}),
\ea
\eeq{eq:wdsub}
is a  solution to this \distancesatsub problem.
\end{Proposition}

\section{Formal Properties of w-systems}~\label{sec:formalProperties}
We now state some interesting formal properties  about w-systems. 
 Word {\em Property} denotes the results that follow rather immediately  from the  model/optimal model definitions. 
 The concluding section of this paper provides proofs to the propositions that appear in this section.
\begin{Property}\label{prop:one}
Any two w-systems with the same hard theory have the same models.
\end{Property}
Due to this property when stating the results for w-systems that share the same hard theory, we only focus on optimal and min-optimal models. 
\begin{Property}\label{prop:secondempty}
Any model of w-system of the form $(\cH,\emptyset)$ is optimal/min-optimal.
\end{Property}

\begin{Property}\label{prop:remove0}
Optimal/min-optimal models of the following w-systems coincide
\begin{itemize}
\item w-system $\cW$ and 
\item w-system resulting from $\cW$ by dropping all of its w-conditions whose  weight is~$0$.
\end{itemize}
\end{Property}
Thus, the w-conditions, whose weight is $0$ are immaterial and can be removed.
For instance, we can safely simplify sample pw-MaxSAT problem~\eqref{eq:partmsat} and MinMaxPL problem~\eqref{eq:maxpl} by dropping their w-conditions $(\neg a\vee b,0)$.

\begin{Property}\label{prop:samefactor}
Given a positive integer $a$, optimal/min-optimal models of the following w-systems coincide
\begin{itemize}
\item w-system $(\cH,\{(T_1,w_1@l_1),\dots,(T_m,w_m@l_m)\})$ and
\item w-system $(\cH,\{(T_1,a\cdot w_1@l_1),\dots,(T_m,a\cdot w_m@l_m)\})$.
\end{itemize}
\end{Property}

We call a w-system $\cW$ {\em level-normal}, when
 we can construct the sequence of numbers $1,2,\dots,|\level{\cW}|$ from the elements in $\level{\cW}$.
  It is easy to see that we can always adjust levels of w-conditions in $\cW$ to respect such a sequence preserving optimal models  of original w-system~$\cW$. 
 \begin{Proposition}\label{prop:levelnormal}
 Optimal/min-optimal models of the following w-systems coincide
\begin{itemize}
\item w-system $\cW$ and 
\item the level-normal w-system constructed from $\cW$ by replacing each level $l_i$ occurring in its w-conditions with its ascending sequence order number $i$, where we arrange elements in $\level{\cW}$ in a   sequence  in ascending order $l_1,l_2,\dots l_{|\level{\cW}|}$. 
\end{itemize}
\end{Proposition}
Sample MaxPL problem~\eqref{eq:maxpl} is not level normal. Yet, this proposition assures us that it is safe to consider the level-normal w-system 
\beq
(F_1,\{(a,1),(b,1@2),(a\vee\neg b,2), (\neg a\vee b,0)\})
\eeq{eq:levelnormal}
in its place.
 In the sequel we often  assume level-normal w-systems without loss of generality.
  
\begin{Proposition}\label{thm:alloptimal}
For a w-system $\cW=(\cH,\cS)$, if every  level $l\in\level{\cW}$ is such that 
for any distinct models $I$ and~$I'$ of $\cW$
 the  equality
 \beq
 \sum_{B\in\cW_{l}}{ \br{I\models B}}= \sum_{B\in\cW_{l}}{ \br{I'\models B}}
 \eeq{eq:eqcond} 
 holds
then optimal/min-optimal models of w-systems
$\cW$ and  $(\cH,\emptyset)$ coincide. Or, in other words, any model of $\cW$ is also optimal and min-optimal model.
\end{Proposition}
By this proposition, for instance, it follows that optimal models of pw-MaxSAT problem $(F_1,\{(a,1),(b,1)\})$ coincide with its models $\{a\}$ and $\{b\}$ or, in other words,  the problem can be simplified to $(F_1,\emptyset)$.

Let $\cW=(\cH,\cS)$ be  a w-system. For a set $S$ of w-conditions, 
by $\less{\cW}{S}$ we denote the w-system $(\cH,\cS\setminus S)$.

\begin{Proposition}\label{thm:samewcond}
For a w-system $\cW=(\cH,\cS)$, if there is a set $S\subseteq \cS$ of w-conditions all sharing the same level~$l$ 
such that 
for any distinct $\prev{l}$-optimal/min-optimal models~$I$ and $I'$ of $\cW$
(or any distinct  models $I$ and $I'$ of $\cW$ in case $\prev{l}$ is undefined)
 the equality~\eqref{eq:eqcond}, where $\cW_{l}$ is replaced by $S$, holds 
then $\cW$ has the same  optimal/min-optimal models as  $\less{\cW}{S}$. 
\end{Proposition}
It is obvious that this proposition holds with  more restrictive condition when words {\em $\prev{l}$-optimal/min-optimal models} are replaced by {\em models}.
This result provides us with the semantic condition on when it is ``safe'' to drop some w-conditions from the w-system.
By this proposition, for instance, it follows that the optimal models of pw-MaxSAT problem~\eqref{eq:partmsat} coincide with the optimal  models of w-system constructed from~\eqref{eq:partmsat} by dropping its w-conditions  $(a,1)$ and $(b,1)$. 
To summarize, all listed results account to the fact that the optimal models of pw-MaxSAT problem~\eqref{eq:partmsat} and the following pw-MaxSAT problem  coincide
\beq
(F_1,\{(a\vee\neg b,2)\}).
\eeq{eq:partmsatsimp}


Let $\signo{(\cH,\{(T_1,w_1@l_1),\dots,(T_n,w_n@l_n)\})}$ map a w-system into the following  w-system
$(\cH,\{(T_1,-1\cdot w_1@l_1),\dots,(T_n,-1\cdot w_n@l_n)\}).$
The next proposition tells us that min-optimal models and optimal models are close relatives:
\begin{Proposition}\label{prop:relatives}
For a  w-system $\cW$,
the optimal models (min-optimal models) of $\cW$ coincide with the min-optimal models (optimal models) of $\signo{\cW}$. 
\end{Proposition}





\subsection{Eliminating Negative (or Positive) Weights}

We call logics $\cL$ and $\cL'$ {\em compatible} when their vocabularies coincide, in other words, when \hbox{$\sigma_\cL=\sigma_{\cL'}$}. 
Let $\cL$ and $\cL'$ be compatible logics, and
$T$ and $T'$ be theories in these logics, respectively. We call such theories {\em compatible}.
For compatible theories~$T$ and $T'$, we call $T$ and~$T'$  as well as  w-conditions $(T,w@l)$ and $(T',w@l)$ {\em equivalent }
when $sem(T)= sem(T')$.
For example, 
sat-logic  theory~\eqref{eq:abtheory} 
  over vocabulary $\{a,b\}$ is equivalent to lp-logic  theory~\eqref{ex:slp} over $\{a,b\}$.


The following proposition captures an apparent property of w-systems that equivalent modules and w-conditions may be substituted by each other without changing the overall semantics of the system. 
\begin{Property}\label{propo:equivalent}
Models and optimal/min-optimal models of  w-systems 
$$(\{T_1,\dots,T_n\},\{B_1,\dots,B_m\}) \hbox{ and } (\{T'_1,\dots,T'_n\},\{B'_1,\dots,B'_m\})$$ 
coincide when
(i) $T_i$ and $T'_i$ ($1\leq i\leq n$) are equivalent theories, and 
(ii) $B_i$ and $B'_i$ ($1\leq i\leq m$)  are equivalent $w$-conditions.
\end{Property}

For a theory $T$ of logic $\cL$, we call a theory $\overline{T}$ in logic~$\cL'$, compatible to~$\cL$, 
{\em complementary} when 
 $sem(\overline{T}) = Int(\sigma_\cL) \setminus sem(T)$.
For example, in case of  pl-logic, theories $F$ and $\neg F$ are complementary.
Similarly, a theory $(\neg a\wedge \neg b)\vee (a\wedge b)$  in pl-logic over vocabulary $\{a,b\}$ is complementary to theory~\eqref{ex:slp} 
in lp-logic over $\{a,b\}$. It is easy to see that given a theory in any logic we can always find, for instance,  a pl-logic or sat-logic theory complementary to it. Yet, given a theory in some arbitrary logic we may not always find a theory complementary to it in the same logic.
For example, 
consider vocabulary $\{a,b\}$ and a wc-theory $a\wedge b$. There is no complementary wc-theory to it over vocabulary $\{a,b\}$.

Let $(T,w@l)$ be a w-condition; consider the following definitions:
$$
\ba{l}
\signp{(T,w@l)}=\begin{cases}
  (T,w@l)&\hbox{when $w\geq 0$, otherwise }  \\
  (\overline{T},-1\cdot w@l)   \\
\end{cases}\\
~\\
\signm{(T,w@l)}=\begin{cases}
  (T,w@l)&\hbox{when $w\leq 0$, otherwise }  \\
  (\overline{T},-1\cdot w@l)   \\
\end{cases}
\ea
$$
where $\overline{T}$ denotes some theory complementary to $T$.
It is easy to see that $\signp{}$ (and $\signm{}$) forms a family of mappings. Applying a member in this family to a w-condition  always results in a w-condition with nonnegative  and nonpositive weights respectively.

For a w-system $(\cH,\{B_1,\dots,B_m\})$, 
we define 
\beq
\ba{l}
\signp{(\cH,\{B_1,\dots, B_n\})}=(\cH,\{\signp{B_1},\dots, \signp{B_n}\})\\
\signm{(\cH,\{B_1,\dots, B_n\})}=(\cH,\{\signm{B_1},\dots, \signm{B_n}\}).\\
\ea
\eeq{eq:plusminus}
The following proposition tells us that negative/positive weights within w-systems may be eliminated in favour of the opposite sign  when theories complementary to theories of w-conditions are found.
\begin{Proposition}\label{prop:signpsignm}
Optimal/min-optimal models of  w-systems  $\cW$, $\signp{\cW}$,   $\signm{\cW}$  coincide.
\end{Proposition}

The result above can be seen as a consequence of the following  proposition:
\begin{Proposition}\label{prop:signpsignm2}
Optimal/min-optimal models of  w-systems  
\begin{itemize}
    \item $(\cH,\{(T,w@l)\}\cup\cS)$ and
    \item {$(\cH,\{(\overline{T},-1\cdot w@l)\}\cup\cS)$}
\end{itemize}
coincide.
\end{Proposition}

Proposition~\ref{prop:signpsignm} suggests that in case of significantly expressive logic the presence of both negative and positive weights in w-conditions is nearly a syntactic sugar. 
Let us illustrate the applicability of this result in the realm of optimization programs. First, we say that 
\begin{itemize}
    \item a weak constraint~\eqref{eq:wc} is {\em  positively-singular} if either  its weight $w\geq 0$  or $m=1$;
    \item a weak constraint~\eqref{eq:wc} is {\em  negatively-singular} if either  its weight $w\leq 0$  or $m=1$.
\end{itemize}
For example, the only weak constraint/wc-logic w-condition  of  o-program~\eqref{eq:sampleop} follows
$$(a\wedge\neg b, -2@1).$$
This weak constraint is negatively-singular. Let us denote it by $C_1$.

Given a positively-singular weak constraint/wc-logic w-condition $B=(T,w@l)$, it is easy to see that a  mapping 
$$B^\uparrow=\begin{cases}
  B&\hbox{when $w\geq 0$, otherwise }  \\
  (\neg a,-1 \cdot w@l) &\hbox{when $T$ has the form $a$}  \\
  (a,-1 \cdot w@l) &\hbox{when $T$ has the form $\neg a$}  \\
\end{cases}
$$
is in the $\signp{B}$ family. 
Given a negatively-singular weak constraint/wc-logic w-condition~$B$, 
$B^\uparrow$ mapping is defined as above by replacing $w\geq 0$ condition with $w\leq 0$.
For instance, $C_1^\uparrow=C_1$. 
It is easy to see that this  mapping 
is in the $\signm{B}$ family. 

Similarly, given a positively-singular weak constraint/wc-logic w-condition $B$ of the form~\eqref{eq:wcwc}, it is easy to see that a  mapping 
$$B^{sat}=\begin{cases}
  \Big(~\eqref{eq:wcwcc},-1\cdot w@l\Big)
  &\hbox{when $w\geq 0$, otherwise }  \\
  B &  \\
\end{cases}
$$
is in the $\signm{B}$ family. 
Given a negatively-singular weak constraint/wc-logic w-condition~$B$ of the form~\eqref{eq:wcwc}, 
$B^{sat}$ mapping is defined as above by replacing $w\geq 0$ condition with $w\leq 0$.
It is easy to see that this  mapping 
is in the $\signp{B}$ family. 
Note that the resulting w-condition of $(\cdot)^{sat}$ mapping is in sat-logic. 
As an example, consider negatively-singular w-condition $C_1$,
$$
C_1^{sat}=(\neg a\vee b, 2@1).
$$

We call optimization program {\em  positively-singular} ({\em  negatively-singular}) when all of its w-conditions are { positively-singular} ({negatively-singular}).
For a positively/negatively-singular optimization program $(\Pi,\{B_1,\dots,B_n\})$, 
$$
\ba{lll}
(\Pi,\{B_1,\dots,B_n\})^\uparrow&=&(\Pi,\{B_1^\uparrow,\dots,B_n^\uparrow\}),\\
(\Pi,\{B_1,\dots,B_n\})^{sat}&=&(\Pi,\{B_1^{sat},\dots,B_n^{sat}\}).
\ea
$$
For example, let $\cP_1$ denote negatively-singular o-program~\eqref{eq:sampleop}. Then, $\cP_1^\uparrow=\cP_1$ and $\cP_1^{sat}=(\Pi_1, \{(\neg a\vee b, 2@1)\})$.

Propositions~\ref{prop:op} and~\ref{prop:signpsignm} tell us that optimal answer sets of positively/negatively-singular o-program $\cP$ 
coincide with min-optimal models of w-system  $\cP^\uparrow$.
Also, they tell us that optimal answer sets of  positively/negatively-singular o-program~$\cP$ coincide with min-optimal models of w-system $\cP^{sat}$.

We note that the restriction on an optimization program to be positively/negatively-singular  is not essential.
For example, we  now describe a procedure that given an arbitrary program constructs  positively-singular one.  In particular, given a program for every weak constraint $C$ of the form~\eqref{eq:wc}, whose weight is negative
\begin{itemize}
    \item  adding to its hard fragment a rule of the form
$$
a^C \ar   a_1,\dotsc, a_\ell,\ not\  a_{\ell+1},\dotsc,\ not\  a_m,
$$
where $a_C$ is a freshly introduced atom and 
\item replacing weak constraint $C$ with
$$
\wr a^C [w@l]
$$
\end{itemize} 
produces a positively-singular optimization program.
The answer sets of these two programs  are in one to one correspondence. Dropping freshly introduced atoms $a^C$ from the answer sets of the newly constructed program  results in the answer sets of the original program. This fact is easy to see given the theorem on explicit definitions~\citeps{fer05}. 
Alternatively, we can apply the transformation described above to  every weak constraint $C$ of the form~\eqref{eq:wc}, whose $m>1$. In this case the resulting program is both  positively-singular and  negatively-singular.
\cite{alv18} describes a normalization procedure in this spirit.

\subsection{Eliminating Levels}\label{sec:levelsel}
We call a w-system $\cW$  {\em  (strictly) positive}  when all of its w-conditions have  {\em (positive) nonnegative} weights.
Similarly, we call a w-system $\cW$  {\em (strictly) negative}  when all of its w-conditions have {\em (negative) nonpositive} weights. As we showed earlier the w-conditions with $0$ weights may safely be dropped so as such the difference between, for example,  strictly positive and positive programs is inessential.

We now show that the notion of level in the definition of w-conditions is immaterial from the expressivity point of view, i.e., they can be considered as  syntactic sugar. Yet, they are convenient mechanism for representing what is called hierarchical optimization constraints. It was also shown in practice that it is often of value to maintain hierarchy of optimization requirements in devising algorithmic solutions to search problems with optimization criteria~\citeps{arg09}.
Here we  illustrate that given an arbitrary w-system we can rewrite it using w-conditions of the form $(T,w)$. 
This change simplifies the definition of an optimal model by reducing it to a single condition; indeed, see Definition~\ref{def:optimalmodelwsysSimple}. 
Intuitively, we  adjust weights $w$ across the w-conditions in a way that mimics their distinct levels. A procedure in style was reported by \cite{alv18} for the case of o-programs. In this work, we generalize that result to arbitrary w-systems that immediately makes it applicable to logical frameworks that go beyond logic programs.  
We also provide a formal proof of the result that was missing in the mentioned paper by~\citeauthor{alv18}.

Let pair $\cW=(\cH,\cS)$ be strictly positive level-normal w-system. As illustrated earlier restricting w-systems to being positive can be seen as an inessential restriction; recall Proposition~\ref{prop:signpsignm}.
We define the number $M_i$ ($0\leq i< |\level{\cW}|$) as
$$M_i=\begin{cases}
  1
  &\hbox{when $i= 0$, otherwise }  \\
  \displaystyle{1+\sum_{
(T,w@i)\in \cS  }}w. &  \\
\end{cases}
$$
 Intuitively, this number gives us the upper bound, incremented by 1, for the sum of the values of the weights of the w-conditions of level~$i$ (we identify  $M_0$ with $1$).
We now define the number that serves the role of the factor for adjusting each weight associated with some level. 
For level~$i$ $(1\leq i\leq |\level{\cW}|)$, let
$f_{i}$ be the number   computed as  
$$
f_{i}=
    \displaystyle{\prod_{0\leq j< i} {M_j}}. 
$$
By $\cS^1$ we denote the set of w-conditions constructed from $\cS$ as follows 
\beq
\{(T,f_i\cdot w)\mid (T,w@i)\in\cS\}
\eeq{eq:relofs}
By $\cW^1$ we denote the w-system resulting from replacing $\cS$ with $\cS^1$.

\begin{Proposition}\label{proposition:one}
Optimal/min-optimal models of strictly positive level-normal w-systems $\cW=(\cH,\cS)$ and $\cW^1=(\cH,\cS^1)$ coincide. 
\end{Proposition}

Recall an example of level-normal w-system~\eqref{eq:levelnormal}. Property~\ref{prop:remove0} tells us that its models and optimal/min-optimal models coincide with these of strictly positive level-normal w-system 
$$(F_1,\{(a,1),(a\vee\neg b,2),(b,1@2)\}).$$
Let us denote this w-system as $\cW_1$.  For  w-system $\cW_1$,
$$
\ba{lll}
M_0=1&M_1=4&\\
&f_1=1&f_2=4\\
\ea
$$
W-system  $\cW_1^1$ follows
$$(F_1,\{(a,1),(a\vee\neg b,2),(b,4)\}).$$
Proposition~\ref{proposition:one} tells us that optimal/min-optimal models of w-systems $\cW_1$ and $\cW_1^1$ coincide.

\section{Optimization Programs as pw-MaxSAT/pw-MinSAT Problems}\label{sec:opmaxsat}

Logic programs under answer set semantic and propositional formulas are closely related (see, for instance, work by~\cite{lier16a} for an overview of translations). For example, for so called ``tight'' programs a well known completion procedure~\citeps{cla78} transforms a logic program into a propositional formula so that the answer sets of the former coincide with the models of the later. Once this formula is clausified the problem becomes a SAT problem. 
For nontight programs extensions of completion procedure are available~\citeps{lin02,jan06a}. Some of those extensions introduce auxiliary atoms. Yet, the appearance of these atoms is inessential as models of resulting formulas are in one to one correspondence with original answer sets. The later can be computed from the former by dropping the auxiliary atoms.
The bottom line is that a number of known translations from logic programs to SAT exists. Numerous answer set solvers, including but not limited to {\sc cmodels}~\citeps{giu06} and {\sc lp2sat}~\citeps{jan06a}, rely on this fact by translating a given logic program into a SAT formula and then applying SAT solvers for computing models/answer sets.
For a logic program $\Pi$ over vocabulary $\sigma$  (that we identify with a module in lp-logic), by  $F_\Pi$ we denote a SAT formula, whose models coincide with these of~$\Pi$. 
For example, recall that $F_1$ and $\Pi_1$ denote sat-formula~\eqref{eq:abtheory} and logic program~\eqref{ex:slp}. Formula $F_1$ forms one of the possible formulas $F_{\Pi_1}$. In fact, $F_1$ corresponds to the clausified completion of program $\Pi_1$, which has the form $$(a\leftrightarrow \neg b)\wedge (b\leftrightarrow \neg a).$$

In previous sections we illustrated how multiple levels and negative weights in w-systems/singular optimization programs can be eliminated in favor of a single level and positive weights.
Thus, without loss of generality we consider here singular optimization programs with a single level. The following result is a consequence of Propositions~\ref{prop:relatives} and~\ref{prop:signpsignm}.
\begin{Proposition}\label{prop:translation}
Optimal answer sets of a positive-singular o-program $(\Pi,\{B_1,\dots,B_m\})$, whose all conditions are of the same level $1$,
coincide with optimal models of pw-MaxSAT problem 
$\signo{(F_{\Pi},\{B_1^{sat},\dots,B_n^{sat}\})}$.
\end{Proposition}
The next proposition follows from Proposition~\ref{prop:signpsignm}.
\begin{Proposition}\label{prop:translation2}
Optimal answer sets of a negatively-singular o-program $(\Pi,\{B_1,\dots,B_m\})$, whose all conditions are of the same level $1$,
coincide with optimal models of pw-MinSAT problem 
${(F_{\Pi},\{B_1^{sat},\dots,B_n^{sat}\})}$.
\end{Proposition}
This result tells us, for example, that optimal answer sets of optimization program~\eqref{eq:sampleop} coincide with optimal models of pw-MinSAT problem~\eqref{eq:pwminsat}. 

Propositions~\ref{prop:translation} and~\ref{prop:translation2} tells us how to utilize MaxSAT/MinSAT solvers for finding optimal answer sets of a program in similar ways as SAT solvers are currently utilized for finding answer sets of logic programs as exemplified by such answer set solvers as {\sc cmodels} or {\sc lp2sat}.

\section{Proofs}
We often show results for optimal models only, as the arguments for min-optimal models follow the same lines.
Given  recursive Definition~\ref{def:optimalmodelwsys} of $l$-(min)-optimal models,  inductive argument is a common technique in proof construction about properties of such models. Below, we refer to the {\em induction on levels of a considered  w-system $\cW$}, where we assume elements in $\level{\cW}$ to be arranged in the descending order $m_1,\dots m_n$ ($n=|\level{\cW}|$); so that the base case is illustrated for the greatest level $m_1$, whereas inductive hypothesis is assumed for level $m_i$ and then illustrated to hold for level $m_{i+1}$. Note how, 
$\prev{m_{i+1}}=m_{i}$.

\subsection{Proofs for Formal Results in Section~\ref{sec:wams}}
An inductive proof on levels of $\cW$ relying on the definition of $l$-optimal models suffices for Lemma~\ref{lem:loptimal}. We omit it for its simplicity. 
Lemma~\ref{lem:loptimal2} follows immediately from Lemma~\ref{lem:loptimal} and an observation that for a pair $I$, $I'$ of $l$-optimal/$l$-min-optimal models of w-system~$\cW$ the following equality holds
\beq \displaystyle{ {\sum_{B\in\cW_l}{ \br{I\models B}}}} =
\displaystyle{ {\sum_{B\in\cW_l}{ \br{I'\models B}}}}.
\eeq{eqlevel}
This observations stems from the definition of $l$-optimal/$l$-min-optimal models. There is one more lemma that we find of use in proving Proposition~\ref{prop:eqdefs}. The claim of the lemma can be shown by induction on levels of $\cW$.
\begin{Lemma}\label{lem:propertylmodels3}
For a w-system $\cW=(\cH,\cS)$, level $l'\in\level{\cW}$, and a model $I'$ of~$\cW$, 
if $I$ is $l'$-optimal model of~$\cW$ such that for all levels $l\geq l'$
equality~\eqref{eqlevel} holds
then  $I'$ is an $l'$-optimal model of $\cW$ as well.
\end{Lemma}

\begin{proof}[Proof of Proposition~\ref{prop:eqdefs}]
We consider the relation between Definitions~\ref{def:optimalmodelwsysASP} and~\ref{def:optimalmodelwsys}. We will illustrate the equivalence between optimal models defined in different terms. We omit an argument for min-optimal models as it follows these lines. First, we illustrate that any optimal model per Definition~\ref{def:optimalmodelwsysASP} is also an optimal model per Definition~\ref{def:optimalmodelwsys}. We call this direction {\em Left-to-right}.
Second, we illustrate that any optimal model per Definition~\ref{def:optimalmodelwsys} is also an optimal model per Definition~\ref{def:optimalmodelwsysASP}. We call this direction {\em Right-to-left}.

Left-to-right.
Per Definition~\ref{def:optimalmodelwsysASP} model $I^*$ of w-system $\cW$ is optimal if there is no model $I'$ of~$\cW$ that max-dominates $I^*$. Consider such model $I^*$. We now illustrate that $I^*$ is also an optimal model of $\cW$ per Definition~\ref{def:optimalmodelwsys}. In other words, $I^*$  is $l$-optimal model for every level $l\in \level{\cW}$. By contradiction. Assume the later statement is not the case: there is a level  $l'\in \level{\cW}$ such that~$I^*$ is not an $l'$-optimal model; take $l$ to be the greatest value in $\level{\cW}$ that satisfies this property.
Consequently,
$I^*$ does not satisfies equation
\beq
I^*=\displaystyle{ \argmax_{I} {\sum_{B\in\cW_l}{ \br{I\models B}}}}
\eeq{eq:argmax}
where
\begin{itemize}
    \item $I$ ranges over models of $\cW$ if $l$ is the greatest level in $\level{\cW}$, 
\item $I$ ranges over $\prev{l}$-optimal models of $\cW$, otherwise.
\end{itemize}

Case 1.  $l$ is the greatest level in $\level{\cW}$. There is a model $I'$ of $\cW$ such that
\beq \displaystyle{ {\sum_{B\in\cW_l}{ \br{I'\models B}}}} > 
\displaystyle{ {\sum_{B\in\cW_l}{ \br{I^*\models B}}}}.
\eeq{eq:eqcond2}

Consequently, $I'$ max-dominates $I^*$. We derive a contradiction.

Case 2.  $l$ is not the greatest level in $\level{\cW}$. Due to the choice of $l$, there is a $\prev{l}$-optimal model~$I'$ of $\cW$  such that inequality~\eqref{eq:eqcond2} holds.
By Lemmas~\ref{lem:loptimal} and~\ref{lem:loptimal2}, for every level $l''\in\level{\cW}$ such that  $l''\geq\prev{l}$
equality~\eqref{eqlevel} holds, where we replace $l$ by $l''$ and $I$ by $I^*$. Consequently, $I'$ max-dominates~$I^*$. We derive a contradiction.

Right-to-left.
Per Definition~\ref{def:optimalmodelwsys} model $I^*$ of w-system $\cW$ is optimal if it is $l$-optimal for every level $l\in\level{\cW}$.
Consider such model $I^*$. We now illustrate that $I^*$ is also an optimal model of $\cW$ per Definition~\ref{def:optimalmodelwsysASP}.
In other words,  there is no model $I'$ of $\cW$ that max-dominates $I^*$. 
By contradiction. Assume the later statement is not the case: 
there is a level  $l\in \level{\cW}$ and a model~$I'$
such that
following conditions are satisfied:
\begin{enumerate}
\item for any level $l'>l$  the following equality holds
$$
\displaystyle{ \sum_{B\in\cW_{l'}}{ \br{I^*\models B}}}
=
\displaystyle{ \sum_{B\in\cW_{l'}}{ \br{I'\models B}}}
$$
\item the inequality~\eqref{eq:eqcond2} holds for level $l$.
\end{enumerate}

Case 1. $l$ is the greatest level. Consequently, $I^*$ does not satisfy equation~\eqref{eq:argmax}. We derive a contradiction.

Case 2. $l$ is not the greatest level. 

Case 2.1 There is level $l''>l$ such that $I'$ is not $l''$-optimal model of $\cW$.
By Lemma~\ref{lem:propertylmodels3} we conclude that
model $I^*$ is not $l''$-optimal model of $\cW$ (indeed, assume the opposite and one derives the contradiction to the statement of the lemma).
We derive a contradiction.

Case 2.2 Model $I'$ is $l''$-optimal model of $\cW$ for every level $l''>l$.
By Lemma~\ref{lem:propertylmodels3} we conclude that
model $I^*$ is  $l''$-optimal model of $\cW$  for every level $l''>l$.
Consequently, both $I'$ and $I^*$ are $\prev{l}$-optimal models of $\cW$.
Condition~\eqref{eq:eqcond2} contradicts the fact that $I^*$ is a solution to equation~\eqref{eq:argmax}.
\end{proof}

\subsection{Proofs for Formal Results in Section~\ref{sec:instances}}
\begin{proof}[Proof of Proposition~\ref{prop:maxsat}]
Consider an arbitrary interpretation $I^*$ over $\sigma$.
We show that $I^*$ is  a solution to MaxSAT problem $F$ if and only if
 $I^*$ is an optimal model of  w-system $(T_\sigma,\{(C,{1})\mid C\in F\})$. Note how the considered w-system is of a special form so that Definition~\ref{def:optimalmodelwsysSimple} is applicable.
Interpretation $I^*$
is a solution to MaxSAT problem $F$ if and only if 
$$
I^*=arg \max_{I}{\sum_{C\in F}{\br{I\models C}}},
$$
where $I$ ranges over all interpretations over $\sigma$.
Interpretation $I^*$ is
 an optimal model of  w-system $(T_\sigma,\{(C,{1})\mid C\in F\})$
if and only if 
$$
I^*=arg \max_{I}{\sum_{B\in \{(C,{1})\mid C\in F\}}{\br{I\models B}}},
$$
where $I$ ranges over all models of $T_\sigma$, which are all interpretations over $\sigma$.

Note how taking the definitions~\eqref{eq:isat} and~\eqref{eq:isatorig} into account we conclude that for any interpretation~$I$ over $\sigma$,
$$
\sum_{B\in \{(C,{1})\mid C\in F\}}{\br{I\models B}}=\sum_{C\in F}{\br{I\models C}}.
$$
\end{proof}
Proofs of Propositions~\ref{prop:wmaxsat} and~\ref{prop:pwmaxsat} follow the logic of the presented proof of Proposition~\ref{prop:maxsat}. In proofs of these propositions, we have to rely on the fact that given a weighted MaxSAT/MinSAT problem $P$ over $\sigma$ and any interpretation $I$ over $\sigma$,
$$
\sum_{(C,w)\in P}{\br{I\models (C,w)}}=\sum_{(C,w)\in P}{w\cdot \br{I\models C}}.
$$

\begin{proof}[Proof of Proposition~\ref{prop:op}]
This proposition considers the relation between 
\begin{itemize}
    \item an optimization logic program $(\Pi,W)$ over $\sigma$ and
\item  w-system $\big(\Pi,\{(D,{w@l})\mid D[w@l]\in W\}\big)$
-- where~$\Pi$ is an lp-logic module and pairs of the form $(D,{w@l})$ are  wc-logic w-conditions.
\end{itemize}
It is important to note that
 for any model $J$ and any level $k$, the following equality holds
$$
\sum_{B\in \{(D,{w@{k}})\mid D[w@{k}]\in W\}} {\br{J\models B}}=
\sum_{D[w@k]\in W} w \cdot {\br{J\models D}}.
$$
With this observation, the argument is straightforward
given Definition~\ref{def:optimalmodelwsysASP} of min-optimal models of w-systems.
\end{proof}

\subsection{Proofs for Formal Results in Section~\ref{sec:qual}}

\begin{proof}[Proof of Proposition~\ref{prop:minone}]
Consider an arbitrary interpretation $I^*$ over $\sigma$.
We show that $I^*$ is  a solution to 
 \minone problem $F$ over $\sigma$ characterized by~$\xi$
 if and only if
 $I^*$ is an optimal model of  w-system
 \hbox{$\cW=(F,\{(\neg a,1)\mid a\in\xi\})$}.
Interpretation $I^*$
is a solution to \minone problem $F$ if and only if 
\beq
I^*=arg \min_{I}{|I\cap\xi|}=arg \min_{I}{\sum_{a\in I\cap\xi} 1}= arg \max_{I}{\sum_{a\in\xi \hbox{ and }a\not\in I\cap\xi} 1},
\eeq{eq:ccard1}
where $I$ ranges over models of $F$.
 Note how the considered w-system is of a special form so that Definition~\ref{def:optimalmodelwsysSimple} is applicable.
 Interpretation $I^*$ is an optimal model of  w-system
 \hbox{$\cW$} if and only if
\beq
I^*=\displaystyle{ \argmax_{I} {\sum_{a\in\xi}{ \br{I\models \neg a}}}= arg \max_{I}{\sum_{a\in\xi \hbox{ and }a\not\in I\cap\xi} 1}
},
\eeq{eq:ccard2} 
where $I$ ranges over models of~$\cW$. Right hand sides of conditions~\eqref{eq:ccard1}  and~\eqref{eq:ccard2} coincide, which concludes the proof.
\end{proof}

\begin{proof}[Proof of Proposition~\ref{prop:minonesub}]
Consider an arbitrary interpretation $I^*$ over $\sigma$.
Take $I^*$ to be an optimal model of  w-system~\eqref{eq:wminonesub},
 where
  \beq a_1,a_2,\dots,a_n
  \eeq{eq:permutation}
  is a permutation of $\xi$.
We show that $I^*$ is  a solution to 
 \minonesub problem $F$ (over $\sigma$ characterized by~$\xi$).
By contradiction. Assume the last claim is not the case. Given
that $I^*$ is an optimal model of  w-system~\eqref{eq:wminonesub}, it is hence a model of $F$.
Thus, for  $I^*$ not to be  a solution to  \minonesub problem $F$ there must exist model $I$ of $F$ such that $I\cap\xi\subset I^*\cap\xi$.
We now show that $I$ max-dominates $I^*$, which will lead us to contradiction with the assumption that $I^*$ is an optimal model.
Take the rightmost occurrence of atoms $a_i$ in  permutation~\eqref{eq:permutation} such that 
$a_i\in (I^*\cap\xi)\setminus I$.  From the fact that $I\cap\xi$ is a subset of $I^*\cap\xi$, it follows for any $j>i$, $a_j$ in permutation~\eqref{eq:permutation} is such that either $a_j\in I$ and $a_j\in I^*$ or $a_j\not\in I$ and $a_j\not\in I^*$.
In other words, for any $j>i$, $\br{I\models \neg a_j}=\br{I^*\models \neg a_j}$; and for $j$, $\br{I\models \neg a_j}=1$ and $\br{I^*\models \neg a_j}=0$.
We just illustrated that for level $i$ among the levels of the considered w-system the conditions for $I$ max-dominating $I^*$ are satisfied.
\end{proof}

\begin{proof}[Proof of Proposition~\ref{prop:minonesub2}]
We consider the notation exactly as in the ''let''-statement of the statement of the proposition.
By contradiction. Assume that $I$ is not an optimal model of w-system~\eqref{eq:wminonesub}. Then there is a model $I'$ of $F$ that max-dominates $I$.
Hence, there is a level $l$ (an index in permutation) so that (a) for any level $l'>l$, $\br{I\models \neg a_{l'}}=\br{I'\models \neg a_{l'}}$ and
(b)  for  level $l$, $\br{I'\models \neg a_{l}}>\br{I\models \neg a_{l'}}$.
By (b), we conclude (c) $a_l$ is such that $a_l\in I$ and $a_l\not \in I'$. From (a) and the fact that the considered permutation
orders atoms of $\xi$ that occur in $I$ prior to the atoms of $\xi$ that do not occur in $I$,  
we conclude that (d) there is no atom in $\xi$ that is not in $I$ but is in $I'$. From (c) and (d) it follows that $I'\cap\xi\subset I$. We derive a contradiction.
\end{proof}

\begin{proof}[Proof of Proposition~\ref{prop:distancesat}]
Consider an arbitrary interpretation $I^*$ over $\sigma$.
We show that $I^*$ is  a solution to 
 \distancesat problem $F$ over $\sigma$ characterized by~$\hat{I}$
 if and only if
 $I^*$ is an optimal model of  w-system~\eqref{eq:done}.
Per definition, interpretation $I^*$
is a solution to \distancesat problem $F$ if and only if 
$$
I^*=\argmin_{I}{|(\hat{I}\setminus I)\cup(I\setminus \hat{I})|},
$$
where $I$ ranges over models of $F$. With the use of Venn diagrams it is easy to see that we may rewrite this expression as
$$
I^*=\argmax_{I}{|\big(\hat{I}\cap I\big)\cup\big((\sigma\setminus I) \cap (\sigma\setminus \hat{I})\big)|}
= \argmax_{I}{|\big(\hat{I}\cap I\big)\cup\big(\sigma\setminus (I\cup \hat{I})\big)|}
.
$$
Obviously, sets $\hat{I}\cap I$ and $\sigma\setminus (I\cup \hat{I})$ are disjoint. 
Consequently,
\beq
I^*=\argmax_{I}{|\hat{I}\cap I|+|\sigma\setminus (I\cup \hat{I})|}=
\argmax_{I}{
\sum_{a\in \hat{I} \hbox{ and } a\in I}{1}  +
\sum_{a\in \sigma \hbox{ and } a\not\in \hat{I} \hbox{ and } a\not \in I}{1}  
}.
\eeq{eq:ccard3}

 Note how the considered w-system is of a special form so that Definition~\ref{def:optimalmodelwsysSimple} is applicable.
 Interpretation $I^*$ is an optimal model of  w-system~\eqref{eq:done} if and only if
\beq
I^*=\displaystyle{ \argmax_{I} {\sum_{a\in\hat{I}}{ \br{I\models a}}+\sum_{a\in\sigma\setminus\hat{I}}{ \br{I\models\neg a}}  }
=
\argmax_{I}{
\sum_{a\in \hat{I} \hbox{ and } a\in I}{1}  +
\sum_{a\in \sigma \hbox{ and } a\not\in \hat{I} \hbox{ and } a\not \in I}{1}  
}
},
\eeq{eq:ccard4} 
where $I$ ranges over models of~$F$. Right hand sides of conditions~\eqref{eq:ccard3}  and~\eqref{eq:ccard4} coincide, which concludes the proof.
\end{proof}

\begin{proof}[Proof of Proposition~\ref{prop:distancesatsub}]
Consider an arbitrary interpretation $I^*$ over $\sigma$.
Take $I^*$ to be an optimal model of  w-system~\eqref{eq:wdsub},
 where
  \beq a_1,a_2,\dots,a_n
  \eeq{eq:permutation2}
  is a permutation of $\sigma$.
We show that $I^*$ is  a solution to 
 \distancesatsub problem $F$ (over $\sigma$ characterized by~$\hat{I}$).
By contradiction. Assume the last claim is not the case. Given
that $I^*$ is an optimal model of  w-system~\eqref{eq:wdsub}, it is hence a model of $F$.
Thus, for  $I^*$ not to be  a solution to  \distancesatsub problem $F$ there must exist model $I$ of $F$ such that 
$$\Big((\hat{I}\setminus I)\cup(I\setminus \hat{I})\Big) \subset
\Big((\hat{I}\setminus I^*)\cup(I^*\setminus \hat{I})\Big).
$$
To satisfy this condition it follows that
$I$ is such that 
\begin{itemize}
    \item[(i)] it contains all elements from $\hat{I}\cap I^*$ and 
    \item[(ii)] it contains no elements that are not in $\hat{I}\cup I^*$.
\end{itemize}

We now show that $I$ max-dominates $I^*$, which will lead us to contradiction with the assumption that $I^*$ is an optimal model.
Take the rightmost occurrence of atoms $a_i$ in  permutation~\eqref{eq:permutation2} such that 
either
\begin{itemize}
    \item $a_i\in I$ and $a_i\in \hat{I}$ and $a_i\not \in I^*$ (and, hence $\br{I\models a_j}=1$ and $\br{I^*\models  a_j}=0$)
or
\item $a_i\not\in I$ and $a_i\not\in  \hat{I}$ and $a_i \in I^*$ (and, hence $\br{I\models \neg a_j}=1$ and $\br{I^*\models \neg a_j}=0$).
\end{itemize}
Given the choice of index $i$ above, we analyze the possibilities for
all occurrences  $a_j$ with $j>i$ in permutation~\eqref{eq:permutation2}. They  are such that either 
\begin{itemize}
\item $a_j\in I$ and $a_j\in \hat{I}$ and $a_j \in I^*$ (and, hence  $\br{I\models a_j}=1=\br{I^*\models  a_j}=1$ ), or
\item $a_j\in I$ and $a_j\not\in \hat{I}$ and $a_j \in I^*$ (and, hence  $\br{I\models \neg a_j}=0=\br{I^*\models \neg a_j}=0$ ), or
\item $a_j\in I$ and $a_j\not\in \hat{I}$ and $a_j \not\in I^*$ (this is an impossible case due to condition (ii)), or
\item $a_j\not\in I$ and $a_j\in \hat{I}$ and $a_j \in I^*$ (this is an impossible case due to condition (i)), or
\item $a_j\not\in  I$ and $a_j\in  \hat{I}$ and $a_j \not\in I^*$ (and, hence  $\br{I\models  a_j}=0=\br{I^*\models  a_j}=0$ ), or
\item $a_j\not\in  I$ and $a_j\not\in  \hat{I}$ and $a_j \not\in I^*$ (and, hence  $\br{I\models \neg a_j}=0=\br{I^*\models \neg a_j}=0$).
\end{itemize}
The conditions within the parenthesis in every considered case are sufficient to conclude that for  for level $i$ among the levels of the considered w-system the conditions for $I$ max-dominating $I^*$ are satisfied.
\end{proof}
\subsection{Proofs for Formal Results in Section~\ref{sec:formalProperties}}

Proof of Proposition~\ref{prop:levelnormal} follows from the fact that the numeric value of any level itself is inessential in the key computations associated with establishing optimal models. Rather, the order of levels with respect to greater relation  matters (recall the definition of $\prev{(\cdot)}$ operation). It is easy to see that changing levels of the w-conditions using the procedure described in Proposition~\ref{prop:levelnormal} preserves original order of the levels with respect to greater relation.

Proposition~\ref{thm:alloptimal} follows immediately from  Proposition~\ref{thm:samewcond}.
We show that Proposition~\ref{thm:samewcond} holds. 

\begin{proof}[Proof of Proposition~\ref{thm:samewcond}]
Let $l$ be the level so that 
there is a set $S\subseteq \cS_l$ of w-conditions such that 
for any distinct $\prev{l}$-optimal models $I$ and $I'$ of $\cW$ the following equality holds
 $$\sum_{B\in S}{ \br{I\models B}}= \sum_{B\in S}{ \br{I'\models B}}.$$
 Let us denote the number associated with the sum in this equality by letter $c$, i.e., $\displaystyle{c=\sum_{B\in S}{ \br{I\models B}}}$ (where $I$ is any $\prev{l}$-optimal model of $\cW$). 

 Note that for every level $l'>l$,  $\cW_{l'}=\less{\cW}{S}_{l'}$. Using simple inductive argument on levels of $\cW$ that are greater than $l$  suffices to show that $l'$-optimal models of  $\cW$ and $\less{\cW}{S}$ coincide. Consequently,  
$\prev{l}$-optimal models of  $\cW$ and $\less{\cW}{S}$ coincide.

We now show that {\em $l$-optimal models of $\cW$ and $\less{\cW}{S}$ coincide}.
We have to illustrate that equation~\eqref{eq:condeqlmin}
and
$$
I^*=\displaystyle{ \argmax_{I} {\sum_{B\in{\cW_l\setminus S}}{ \br{I\models B}}}}
$$
have the same solutions, when $I$ ranges over $\prev{l}$-optimal models of  $\cW$.
This is apparent from the fact that the right hand side (RHS) of the  equation~\eqref{eq:condeqlmin}
can be rewritten as
$$
\ba{l}
\displaystyle{\argmax_{I} \Big(
{\sum_{B\in{\cW_l\setminus S}}{ \br{I\models B}}}+ 
{\sum_{B\in S}{ \br{I\models B}}
\Big)}} = \\
\displaystyle{\argmax_{I}
\Big(
\sum_{B\in{\cW_l\setminus S}}{ \br{I\models B}}+ c
\Big)} =\\
\displaystyle{\argmax_{I}
{\sum_{B\in{\cW_l\setminus S}}{ \br{I\models B}}}}.
\ea
$$
Using simple inductive argument on levels of $\cW$ that are less than $l$  suffices to show that $l'$-optimal models of  $\cW$ and $\less{\cW}{S}$ coincide, where $l'$ is any level less than $l$. 
\end{proof}

 \begin{proof}[Proof of Proposition~\ref{prop:relatives}]
To show that 
the optimal models  of $\cW$ coincide with the min-optimal models  of $\signo{\cW}$, it is sufficient to show that for any level in $\level{\cW}$, $l$-optimal models of $\cW$ coincide with $l$-min-optimal models of  $\signo{\cW}$.
We first note that models of $\cW$ and $\signo{\cW}$ coincide.
By induction on levels of~$\cW$.

Base case. $l$ is the greatest level.
A model $I^*$ of $\cW$ is {\em $l$-optimal} if and only if
$I^*$ satisfies equation~\eqref{eq:argmax},
where $I$ ranges over models of $\cW$.
It is easy to see that we can rewrite this equation as 
$$
I^*=\displaystyle{ \argmin_{I} {\sum_{B\in\cW_l} -1\cdot { \br{I\models B}}}
}.
$$
It immediately follows from the construction of $\signo{\cW}$ that this equation can be rewritten as
$$
I^*=\displaystyle{ \argmin_{I} {\sum_{B\in\signo{\cW}_l} { \br{I\models B}}}
}.
$$
Thus $I^*$ is an $l$-min-optimal model of $\signo{\cW}$ as the equation above is exactly the one from the definition of $l$-min-optimal models of  $\signo{\cW}$; plus recall that models of $\cW$ and $\signo{\cW}$ coincide.

Inductive case argument follows similar lines.
\end{proof} 

Proposition~\ref{prop:signpsignm} follows immediately from  Proposition~\ref{prop:signpsignm2}.
We show that Proposition~\ref{prop:signpsignm2} holds.

\begin{proof}[Proof of Proposition~\ref{prop:signpsignm2}]
We show the argument for optimal models for the case of w-systems
$\cW=(\cH,\{(T,w@l)\}\cup\cS)$
and $\cW'= (\cH,\{(\overline{T},-1\cdot w@l)\}\cup\cS)$. 
The argument  
for the case of min-optimal models follows the same lines. 

To prove that 
the optimal models  of $\cW$ coincide with the optimal models  of $\cW'$, it is sufficient to show that for any level $l'$ in $\level{\cW}$, $l'$-optimal models of $\cW$ coincide with $l'$-optimal models of~$\cW'$.
We first note that models of $\cW$ and ${\cW'}$ coincide.
By induction on levels of  $\cW$.

Base case. $l'$ is the greatest level.

Case 1. $l'\neq l$. The case is trivial as $\cW_{l'}={\cW'}_{l'}$.

Case 2. $l'= l$. 
A model $I^*$ of $\cW$ (or $\cW'$) is {\em $l$-optimal} model of $\cW$ if and only if
$I^*$ satisfies equation~\eqref{eq:argmax},
where  $I$ ranges over models of $\cW$.
In other words, for any model $I$ of $\cW$  (or $\cW'$) the inequality
\beq
\displaystyle{ 
{\sum_{B\in\cW_l}{ \br{I^*\models B}}}\geq 
{\sum_{B\in\cW_l}{ \br{I\models B}}}
}.
\eeq{eq:name1}
holds.
This inequality can be equivalently rewritten as
\beq
\ba{l}
\displaystyle{ 
{\sum_{B\in\cW_l, B\neq \{(T,w@l)\}}{ \br{I^*\models B}}}
+ \br{I^*\models \{(T,w@l)\}}}
\geq \\
\displaystyle{ 
{\sum_{B\in\cW_l, B\neq \{(T,w@l)\}}{ \br{I\models B}}}
+ \br{I\models \{(T,w@l)\}}
}.
\ea
\eeq{eq:name2.2}

Note that the following condition holds
\beq \cW_l\setminus\{(T,w@l)\}=
\cW'_l\setminus\{(\overline{T},-1\cdot w@l)\}.
\eeq{eq:condall}

We now illustrate that under all possible cases we conclude that inequality~\eqref{eq:name1}, where $\cW_l$ is replaced by $\cW'_l$, holds for arbitrary model $I$ of $\cW$ (or $\cW'$). Consequently, $I^*$ satisfies   
equation~\eqref{eq:argmax},
where $\cW_l$ is replaced by $\cW'_l$
and
$I$ ranges over models of~$\cW'$.
Thus,  $I^*$  is  $l$-optimal model of $\cW'$. 
We also note that the same cases are applicable to illustrate another direction. Namely, from the fact that 
$I^*$ satisfies   
equation~\eqref{eq:argmax},
where $\cW_l$ is replaced by $\cW'_l$
and
$I$ ranges over models of~$\cW'$ we can conclude that $I^*$ satisfies   
equation~\eqref{eq:argmax} where $I$ ranges over models of~$\cW$. 
To support this claim, the arguments of the cases below can be read from bottom to top.

\noindent
Case 2.1: $I^*\models \{(T,w@l)\}$ and $I\models \{(T,w@l)\}$. Or, equivalently, \hbox{$I^*\not \models \{(\overline{T},-1\cdot w@l)\}$} and $I\not \models \{(\overline{T},-1\cdot w@l)\}$. 
Consequently,
\beq
\br{I^*\models \{(\overline{T},-1\cdot w@l)\}}=0=
\br{I \models \{(\overline{T},-1\cdot w@l)\}}.
\eeq{cond1}

Inequality~\eqref{eq:name2.2} can be rewritten as
$$
\ba{l}
\displaystyle{ 
{\sum_{B\in\cW_l, B\neq \{(T,w@l)\}}{ \br{I^*\models B}}}}
+ w 
\geq 
\displaystyle{ 
{\sum_{B\in\cW_l, B\neq \{(T,w@l)\}}{ \br{I\models B}}}
+ w
}.
\ea
$$
Consequently, the following inequality holds
$$
\ba{l}
\displaystyle{ 
{\sum_{B\in\cW_l, B\neq \{(T,w@l)\}}{ \br{I^*\models B}}}
}
\geq 
\displaystyle{ 
{\sum_{B\in\cW_l, B\neq \{(T,w@l)\}}{ \br{I\models B}}}
}.
\ea
$$
From this inequality, it immediately follows that the following holds
\beq
\ba{l}
\displaystyle{ 
{\sum_{B\in\cW'_l, B\neq \{(\overline{T},-1\cdot w@l)\}}{ \br{I^*\models B}}}
+ \br{I^*\models \{(\overline{T},-1\cdot w@l)\}}}
\geq \\
\displaystyle{ 
{\sum_{B\in\cW'_l, B\neq \{(\overline{T},-1\cdot w@l)\}}{ \br{I\models B}}}
+ \br{I\models \{(\overline{T},-1\cdot w@l)\}}
}.
\ea
\eeq{eq:name2.5}
 as  conditions~\eqref{eq:condall} and~\eqref{cond1} hold.
By inequality~\eqref{eq:name2.5},
inequality~\eqref{eq:name1} where $\cW_l$ is replaced by $\cW'_l$ holds.

\noindent
Case 2.2: $I^*\models \{(T,w@l)\}$ and $I\not \models \{(T,w@l)\}$.
 Or, equivalently, \hbox{$I^*\not \models \{(\overline{T},-1\cdot w@l)\}$} and $I \models \{(\overline{T},-1\cdot w@l)\}$. 
Consequently,
\beq
\br{I^*\models \{(\overline{T},-1\cdot w@l)\}}=0\hbox{ and }
\br{I \models \{(\overline{T},-1\cdot w@l)\}}=-w.
\eeq{cond2}

Inequality~\eqref{eq:name2.2} can be rewritten as
$$
\ba{l}
\displaystyle{ 
{\sum_{B\in\cW_l, B\neq \{(T,w@l)\}}{ \br{I^*\models B}}}}
+ w 
\geq 
\displaystyle{ 
{\sum_{B\in\cW_l, B\neq \{(T,w@l)\}}{ \br{I\models B}}}
}.
\ea
$$
Consequently,
$$
\ba{l}
\displaystyle{ 
{\sum_{B\in\cW_l, B\neq \{(T,w@l)\}}{ \br{I^*\models B}}}}
\geq 
\displaystyle{ 
{\sum_{B\in\cW_l, B\neq \{(T,w@l)\}}{ \br{I\models B}}}
}-w.
\ea
$$
From this inequality, it immediately follows that the inequality~\eqref{eq:name2.5} holds
  as  conditions~\eqref{eq:condall} and~\eqref{cond2} hold.
By inequality~\eqref{eq:name2.5},
inequality~\eqref{eq:name1} where $\cW_l$ is replaced by $\cW'_l$ holds.

\noindent
Case 2.3: $I^*\not \models \{(T,w@l)\}$ and $I \models \{(T,w@l)\}$.  Or, equivalently, \hbox{$I^* \models \{(\overline{T},-1\cdot w@l)\}$} and $I\not \models \{(\overline{T},-1\cdot w@l)\}$. 
Consequently,
\beq
\br{I^*\models \{(\overline{T},-1\cdot w@l)\}}=-w\hbox{ and }
\br{I \models \{(\overline{T},-1\cdot w@l)\}}=0.
\eeq{cond3}

Similarly to Case 2.2. we derive that from inequality~\eqref{eq:name2.2}  follows that 
$$
\ba{l}
\displaystyle{ 
{\sum_{B\in\cW_l, B\neq \{(T,w@l)\}}{ \br{I^*\models B}}}-w}
\geq 
\displaystyle{ 
{\sum_{B\in\cW_l, B\neq \{(T,w@l)\}}{ \br{I\models B}}}
}.
\ea
$$
From this inequality, it immediately follows that the inequality~\eqref{eq:name2.5} holds
 as conditions~\eqref{eq:condall} and~\eqref{cond3} hold.
By inequality~\eqref{eq:name2.5},
inequality~\eqref{eq:name1} where $\cW_l$ is replaced by $\cW'_l$ holds.

\noindent
Case 2.4: $I^*\not \models \{(T,w@l)\}$ and $I \not \models \{(T,w@l)\}$. 
Or, equivalently, \hbox{$I^* \models \{(\overline{T},-1\cdot w@l)\}$} and $I \models \{(\overline{T},-1\cdot w@l)\}$. 
Consequently,
\beq
\br{I^*\models \{(\overline{T},-1\cdot w@l)\}}=-w=
\br{I \models \{(\overline{T},-1\cdot w@l)\}}.
\eeq{cond4}
We derive that from inequality~\eqref{eq:name2.2}  follows that 
$$
\ba{l}
\displaystyle{ 
{\sum_{B\in\cW_l, B\neq \{(T,w@l)\}}{ \br{I^*\models B}}}
}
\geq 
\displaystyle{ 
{\sum_{B\in\cW_l, B\neq \{(T,w@l)\}}{ \br{I\models B}}}
}.
\ea
$$
From this inequality, it immediately follows that the inequality~\eqref{eq:name2.5} holds
 as conditions~\eqref{eq:condall} and~\eqref{cond4} hold.
By inequality~\eqref{eq:name2.5},
inequality~\eqref{eq:name1} where $\cW_l$ is replaced by $\cW'_l$ holds.

Inductive case argument follows similar lines.
\end{proof}

\begin{proof}[Proof of Proposition~\ref{proposition:one}]
We focus on the claim of identical optimal models of $\cW$ and~$\cW^1$ (the claim about min-optimal follows the same lines).
 We first note that models of $\cW$ and $\cW^1$ are identical.

Recall
Definition~\ref{def:optimalmodelwsysASP}. 
It is easy to see that our claim follows in case when
for any pair $I$ and $I'$  of models of $\cW$ (or $\cW^1$), $I'$ max-dominates $I$ in $\cW$ if and only if $I'$  max-dominates $I$ in $\cW'$. Take $n$ denote a number of levels in $\cW$.

Right-to-left. Assume that $I'$ max-dominates $I$ in $\cW$.
Per Definition~\ref{def:optimalmodelwsysASP}, {\em 
there exists a level $l\in\level{\cW}$ such that
following conditions are satisfied:
\begin{enumerate}
\item\label{l:cond1.1} for any level $l'>l$  the following equality holds
$$
\displaystyle{ \sum_{B\in\cW_{l'}}{ \br{I\models B}}}
=
\displaystyle{ \sum_{B\in\cW_{l'}}{ \br{I'\models B}}}
$$
\item\label{l:cond2.1} the following inequality holds for level $l$
\beq
\displaystyle{ \sum_{B\in\cW_l}{ \br{I'\models B}}}
>
\displaystyle{ \sum_{B\in\cW_l}{ \br{I\models B}}}
\eeq{eq:cond2.2} 
\end{enumerate}
}
In the left-to-right direction of the proof we refer to the text in italics as {\em max-dominating condition}.

The goal is to illustrate that the following inequality holds 
 \beq
\displaystyle{\sum_{B\in\cW^1}{\br{I'\models B}}
} > \displaystyle{ \sum_{B\in\cW^1}{\br{I\models B}}
},
\eeq{eq:cond3}
which translates into the fact that $I'$ max-dominates $I$ in $\cW^1$.

We will now write two chains  of inequalities for LHS and RHS expressions of~\eqref{eq:cond3} that support the claim of condition~\eqref{eq:cond3}. Note how the last element of each chain amounts to the same expression.  After every chain we provide explanations supporting their less trivial transformations. 

The first chain follows
$$\ba{r}
\displaystyle{ 
\sum_{B\in\cW^1}{\br{I'\models B}}}
=\\
\displaystyle{ 
\sum_{1\leq i\leq {n}} {\big( f_i\cdot \sum_{B\in\cW_i}{\br{I'\models B}}\big)}}=\\
\displaystyle{ 
\sum_{1\leq i< {l}} {\big( f_i\cdot \sum_{B\in\cW_i}{\br{I'\models B}}\big)} 
+ 
 f_{l}\cdot \sum_{B\in\cW_{l}}{\br{I'\models B}}
+ 
\sum_{l< i\leq n} {\big( f_i\cdot \sum_{B\in\cW_i}{\br{{\mathbf I'}\models B}}\big)}} 
=\\
\displaystyle{ 
\sum_{1\leq i< {l}} {\big( f_i\cdot \sum_{B\in\cW_i}{\br{I'\models B}}\big)} 
+ 
 f_{l}\cdot \sum_{B\in\cW_{l}}{\br{I'\models B}}
+ 
\sum_{l< i\leq n} {\big( f_i\cdot \sum_{B\in\cW_i}{\br{{\mathbf I}\models B}}\big)}}\geq\\ 
\displaystyle{ 
\sum_{1\leq i< {l}} {\big(f_i\cdot 0 \big)}+ 
 f_{l}\cdot \sum_{B\in\cW_{l}}{\br{I'\models B}}
+ 
\sum_{l< i\leq n} {\big( f_i\cdot \sum_{B\in\cW_i}{\br{ I\models B}}\big)}}=\\
\displaystyle{ 
 f_{l}\cdot \sum_{B\in\cW_{l}}{\br{{\mathbf I'}\models B}}
+ 
\sum_{l< i\leq n} {\big( f_i\cdot \sum_{B\in\cW_i}{\br{ I\models B}}\big)}}\geq\\
\displaystyle{ 
 f_{l} + f_{l}\cdot \sum_{B\in\cW_{l}}{\br{{\mathbf I}\models B}}
+ 
\sum_{l< i\leq n} {\big( f_i\cdot \sum_{B\in\cW_i}{\br{ I\models B}}\big)}}> \\
\displaystyle{ 
-1+ f_{l} + f_{l}\cdot \sum_{B\in\cW_{l}}{\br{{\mathbf I}\models B}}
+ 
\sum_{l< i\leq n} {\big( f_i\cdot \sum_{B\in\cW_i}{\br{ I\models B}}\big)}}
\ea
$$
The third equality in the chain is due to Condition~\ref{l:cond1.1}.
The appearance of the first inequality in the chain is due to the fact that  
$$\sum_{B\in\cW_i}{\br{I'\models B}}\geq 0$$
as $0$ is ``the most pessimistic'' case for the ``cost'' of w-conditions at the $i$-th level,  when  model $I$ does not satisfy any of the respective w-conditions. 
The appearance of the second inequality in the chain is due to the observation that from the condition~\eqref{eq:cond2.2}  it follows that
 $$
 \ba{clc}
 \displaystyle{\sum_{B\in\cW_{l}}{\br{I'\models B}}}
 &\geq& \displaystyle{ \sum_{B\in\cW_{l}}{\br{I\models B}}} +1
\\
f_{l}\cdot\displaystyle{\sum_{B\in\cW_{l}}{\br{I'\models B}}
}
&\geq& 
\displaystyle{ f_{l}\cdot \big(\sum_{B\in\cW_{l}}{\br{I\models B}}
} +1\big) 
\\
f_{l}\cdot\displaystyle{\sum_{B\in\cW_{l}}{\br{I'\models B}}
}
&\geq& 
\displaystyle{  f_{l}+ f_{l}\cdot \sum_{B\in\cW_{l}}{\br{I\models B}}
}. 
\ea
$$

The second chain follows
$$\ba{r}
\displaystyle{ 
\sum_{B\in\cW^1}{\br{I\models B}}}
=\\
\displaystyle{ 
\sum_{1\leq i< {l}} {\big( f_i\cdot \sum_{B\in\cW_i}{\br{I\models B}}\big)}} 
+ 
 f_{l}\cdot \sum_{B\in\cW_{l}}{\br{I\models B}}
+ 
\sum_{l< i\leq n} {\big( f_i\cdot \sum_{B\in\cW_i}{\br{I\models B}}\big)} 
\leq\\
\displaystyle{ 
\sum_{1\leq i< {l}} {\big(f_i\cdot (M_i -1)\big)} 
+ 
 f_{l}\cdot \sum_{B\in\cW_{l}}{\br{I\models B}}
+ 
\sum_{l< i\leq n} {\big( f_i\cdot \sum_{B\in\cW_i}{\br{I\models B}}\big)}} =
\\
\displaystyle{ 
M_1 -1 + M_1\cdot M_2 -M_1+\dots+
 M_1\cdot...\cdot M_{l-1} - M_1\cdot...\cdot M_{l-2}+}\\
\displaystyle{ 
 f_{l}\cdot \sum_{B\in\cW_{l}}{\br{I\models B}}
+ 
\sum_{l< i\leq n} {\big( f_i\cdot \sum_{B\in\cW_i}{\br{I\models B}}\big)}} 
 =\\
\displaystyle{ 
 -1 + M_1\cdot...\cdot M_{l-1} + 
 f_{l}\cdot \sum_{B\in\cW_{l}}{\br{I\models B}}
+ 
\sum_{l< i\leq n} {\big( f_i\cdot \sum_{B\in\cW_i}{\br{I\models B}}\big)}}
=\\
\displaystyle{ 
 -1 + f_{l} + 
 f_{l}\cdot \sum_{B\in\cW_{l}}{\br{I\models B}}
+ 
\sum_{l< i\leq n} {\big( f_i\cdot \sum_{B\in\cW_i}{\br{I\models B}}\big)}}
\ea
$$
The appearance of the inequality in the chain above is due to the fact that  
$$\sum_{B\in\cW_i}{\br{I\models B}}\leq (M_i -1)$$
as $M_i$ accounts for the sum, incremented by one, of all the  weights in w-conditions of~$\cW_i$, which is ``the most optimistic'' case for the ``cost'' of w-conditions at the $i$-th level  when  model $I$  satisfies all of the respective w-conditions. 

Left-to-right. Assume that $I$ max-dominates $I'$ in $\cW^1$.
Per Definition~\ref{def:optimalmodelwsysASP},
inequality~\eqref{eq:cond3} holds.

The goal is to illustrate that max-dominating condition (see right-to-left direction) holds,
which translates into the fact that $I'$ max-dominates $I$ in $\cW$.
From inequality~\eqref{eq:cond3} we derive that
\beq
\displaystyle{ 
\sum_{1\leq i\leq {n}} \big( f_i\cdot \sum_{B\in\cW_i}{\br{I'\models B}}\big)}
>
\displaystyle{ 
\sum_{1\leq i\leq {n}} \big( f_i\cdot \sum_{B\in\cW_i}{\br{I\models B}}\big)}.
\eeq{eq:condfi}
We now show that there is $l$ among $1..n$ such that  (i)
$$
\displaystyle{ 
 f_l\cdot \sum_{B\in\cW_l}{\br{I'\models B}}}
>
\displaystyle{  
 f_l\cdot \sum_{B\in\cW_l}{\br{I\models B}}}
$$
holds and (ii) for any $l'>l$, 
\beq
\displaystyle{ 
 f_{l'}\cdot \sum_{B\in\cW_{l'}}{\br{I'\models B}}}
=
\displaystyle{  
 f_{l'}\cdot \sum_{B\in\cW_{l'}}{\br{I\models B}}}.
\eeq{eq:equality1}
With that max-dominating condition immediately follows.

By contradiction. Assume no such level $l$ exists such that both conditions (i) and (ii) hold.

Case 1: Assume there is no level $l$ for which condition (i) holds. This immediately contradicts assumption~\eqref{eq:condfi}.

Case 2:  There is $l$ among $1..n$ such that  (i) holds. Take $l$ to be the greatest level among $1..n$ for which (i) holds. Assume that (ii) does not hold. Consequently, there is $l''>l$ such that
\beq
\displaystyle{ 
 f_{l''}\cdot \sum_{B\in\cW_{l''}}{\br{I'\models B}}}
<
\displaystyle{  
 f_{l''}\cdot \sum_{B\in\cW_{l''}}{\br{I\models B}}}.
\eeq{eq:ineq1}
Take $l''$ to be the greatest level with such property. (Consequently, for any level $l'>l''$ equality~\eqref{eq:equality1} holds.)
We now illustrate that from this fact the inequality 
\beq
\displaystyle{ 
\sum_{1\leq i\leq {n}} \big( f_i\cdot \sum_{B\in\cW_i}{\br{I'\models B}}\big)}
<
\displaystyle{ 
\sum_{1\leq i\leq {n}} \big( f_i\cdot \sum_{B\in\cW_i}{\br{I\models B}}\big)}.
\eeq{eq:condfiop}
follows, which contradicts assumption~\eqref{eq:condfi}.
We  now write two chains  of inequalities for  LHS and RHS expressions of~\eqref{eq:condfiop} to support this claim. After every chain we provide explanations supporting their less trivial transformations. 
The first chain follows
$$
\ba{l}
\displaystyle{ 
\sum_{1\leq i\leq {n}} \big( f_i\cdot \sum_{B\in\cW_i}{\br{I'\models B}}\big)}=\\
\displaystyle{ 
\sum_{1\leq i\leq {l''}} \big( f_i\cdot \sum_{B\in\cW_i}{\br{I'\models B}}\big)}=\\
\displaystyle{ 
\sum_{1\leq i\leq {l''-1}} \big( f_i\cdot \sum_{B\in\cW_i}{\br{I'\models B}}\big)}
+ f_{l''}\cdot \sum_{B\in\cW_{l''}}{\br{I'\models B}} \leq\\
\displaystyle{ 
\sum_{1\leq i\leq {l''-1}} \big( f_i\cdot \sum_{B\in\cW_i}{\br{I'\models B}}\big)}
+ f_{l''}\cdot (\sum_{B\in\cW_{l''}}{\br{I\models B}} -1) \leq\\
\displaystyle{ 
\sum_{1\leq i\leq {l''-1}}  {\big(f_i\cdot (M_i-1)\big)}
+ f_{l''}\cdot (\sum_{B\in\cW_{l''}}{\br{I\models B}} -1)} =\\
\displaystyle{ 
-1 +  f_{l''}\cdot \sum_{B\in\cW_{l''}}{\br{I\models B}}
}
\ea
$$
The first equality is due to the choice of $l''$ so that for any level $l'>l''$ equality~\eqref{eq:equality1} holds. The first inequality is due to the observation that from inequality~\eqref{eq:ineq1} the following conditions follow immediately
$$
\ba{l}
\displaystyle{ 
\sum_{B\in\cW_{l''}}{\br{I'\models B}}}
<
\displaystyle{  
\sum_{B\in\cW_{l''}}{\br{I\models B}}}.\\
\displaystyle{ 
\sum_{B\in\cW_{l''}}{\br{I'\models B}}}
\leq
\displaystyle{  
\sum_{B\in\cW_{l''}}{\br{I\models B}}}-1.\\
\ea
$$
The second inequality and the last equality follow the reasoning of the second chain presented in right-to-left direction. 

The second chain follows
$$
\ba{l}
\displaystyle{ 
\sum_{1\leq i\leq {n}} \big( f_i\cdot \sum_{B\in\cW_i}{\br{I\models B}}\big)}=\\
\displaystyle{ 
\sum_{1\leq i\leq {l''}} \big( f_i\cdot \sum_{B\in\cW_i}{\br{I\models B}}\big)}=\\
\displaystyle{ 
\sum_{1\leq i\leq {l''-1}} \big( f_i\cdot \sum_{B\in\cW_i}{\br{I\models B}}\big)}
+ f_{l''}\cdot \sum_{B\in\cW_{l''}}{\br{I\models B}} \geq\\
\displaystyle{ 
f_{l''}\cdot (\sum_{B\in\cW_{l''}}{\br{I\models B}})} >\\
\displaystyle{ -1 +
f_{l''}\cdot (\sum_{B\in\cW_{l''}}{\br{I\models B}})}
\ea
$$
\end{proof}

\section{Conclusions}

 We proposed the extension of abstract modular systems to weighted systems in a way that modern approaches to optimizations stemming from a variety of different logic-based formalisms can be studied in unified terminological ways so that their differences and similarities become clear not only on intuitive but also formal level. 
 
In this article we utilize the presented weighted systems as an abstraction of search-optimization problems and capture such formalisms as MaxSAT and ASP-WC by illustrating that these can be seen as instances of weighted systems.
We present multiple general equivalence transformations, which include interchangeability of logic modules having the same models,
preservation of optimal and min-optimal models when weighted conditions yielding the same
sum for all models are eliminated, and convertibility between optimal and
min-optimal models by multiplying all weights with $-1$. Moreover, transformations
to eliminate either negative or positive weights, map several levels to a single
level by scaling up weights, and turn ASP-WC programs into MaxSAT are developed.
Some similar transformations/normalizations were described as early as~\citeyear{sns02} by \citeauthor{sns02} and later by~\cite{alv18} in the {\em context of ASP-WC}. Here, they are lifted into an abstract settings allowing us the application of the results within different frameworks. Proofs for all of the stated theoretical results are provided.

 We trust that the proposed unifying framework of w-systems will allow  developers of distinct automated paradigms to better grasp similarities and differences of the kind of optimization criteria their paradigms support. In practice, translational approaches are popular in devising solvers. These approaches rely on the relation between automated reasoning paradigms and rather than devising a unique search algorithm for a paradigm of interest propose a translation to another framework so that the solvers for that frameworks can be of use. This work is a stepping stone towards extending these translational approaches with the support for optimization statements. 
 In particular, an immediate and an intuitive future work direction is extending a translational based answer set solver {\sc cmodels} with capabilities to process optimization statements by enabling it to interface with a MaxSAT/MinSAT solver in place of a SAT solver. The formal results of this paper support the validity of this approach (see Proposition~\ref{prop:translation} and~\ref{prop:translation2}). It is also apparent how \minone solvers can be used for finding solutions to \minoneasp problems introduced here. 
 
 In addition, this work served a role of a spring board for a unifying framework studied in~\citep{lie22} that looked into formalisms which mix propositional and non-propositional (e.g., integer) variables in their languages and optimization statements. This supports the claim of explanatory and simplifying utility of this paper.

\noindent
{\bf Acknowledgments}~~~{We would like to acknowledge Mario Alviano, Martin Gebser, Jorge Fandinno, Torsten Schaub, Miroslaw Truszczynski, and Da Shen for valuable discussions related to this work. 
The project was partially supported by NSF grant 1707371.}


\bibliographystyle{tlplike}	
\bibliography{abstractmods-bib}
\end{document}